\documentclass[letterpaper,journal]{IEEEtran}
\usepackage{amsmath,amsfonts}
\usepackage{algorithmic}
\usepackage{algorithm}
\usepackage{array}
\usepackage[caption=false,font=normalsize,labelfont=sf,textfont=sf]{subfig}
\usepackage{textcomp}
\usepackage{stfloats}
\usepackage{url}
\usepackage{verbatim}
\usepackage{graphicx}
\usepackage{cite}
\usepackage{newfloat}
\usepackage{listings}
\usepackage{enumitem}

\usepackage{float} 

\newcommand{\edit}[1]{#1}
\makeatletter

\newcommand{\Rmnum}[1]{\expandafter\@slowromancap\romannumeral #1@}
\makeatother

\usepackage{amsmath,amsfonts}
\usepackage{cancel}
\usepackage{graphicx}
\usepackage{textcomp}
\usepackage{dsfont}
\usepackage{xcolor}
\usepackage{bm}
\usepackage{amsmath}

\usepackage{etoolbox}
\usepackage{ifthen}
\newboolean{showcomments}
\setboolean{showcomments}{true}

\usepackage{verbatim}
\usepackage{booktabs}
\usepackage[mathscr]{euscript}
\usepackage{verbatim}
\usepackage{amssymb,amsmath,color,graphicx, amsbsy, bm}
\usepackage{epsfig}
\usepackage{amsthm}

\usepackage{xcolor}

\usepackage{algorithmic}       
\usepackage{amsmath}
\usepackage{graphicx}
\usepackage{booktabs}
\usepackage[flushleft]{threeparttable}
\usepackage{multirow}
\usepackage{amsthm}
\usepackage{graphicx,booktabs,multirow}
\usepackage{stmaryrd}
\usepackage{multirow,url}
\usepackage{algorithm}
\usepackage{color}

\usepackage{mathtools}

\newtheorem*{kquestion}{Key Question}
\newtheorem{theorem}{Theorem}

\newtheorem{proposition}{Proposition}
\newtheorem{lemma}{Lemma}

\newtheorem{definition}{Definition}
\newtheorem{assumption}{Assumption}

\def\K{\mathcal{K}}
\def\I{\mathcal{I}}
\def\N{\mathcal{N}}
\def\M{\mathcal{M}}

\def\H{\mathcal{H}}

\def\s{\boldsymbol{s}}

\def\As{\mathcal{A}}

\def\task{k}

\newcommand\norm[1]{\lVert#1\rVert}

\def\so{\boldsymbol{s}^{\textsc{O}}}

\def\ba{\boldsymbol{a}^{\textsc{h}}}
\def\ao{\boldsymbol{a}^{\textsc{O}}}

\def\pio{{\pi}^{\textsc{O}}}

\def\ba{\boldsymbol{a}}
\def\bs{\boldsymbol{s}}
\def\bc{\boldsymbol{c}}
\def\bN{\boldsymbol{N}}
\def\bF{\boldsymbol{F}}
\def\bG{\boldsymbol{G}}
\def\bW{\boldsymbol{W}}
\def\bD{\boldsymbol{D}}
\def\bQ{\boldsymbol{Q}}

\def\Ao{\mathcal{A}^{\textsc{O}}}
\def\Do{\mathcal{D}^{\textsc{O}}}

\def\bD{\boldsymbol{D}}
\def\bV{\boldsymbol{V}}

\def\bgamma{\boldsymbol{\gamma}}

\def\q{\boldsymbol{q}}
\def\pii{\boldsymbol{\pi}}

\newcommand{\rev}[1]{#1}
\newcommand{\revb}[1]{#1}
\newcommand{\revbb}[1]{#1}
\newcommand{\revr}[1]{#1}

\newcommand{\com}[1]{}
\newcommand{\comg}[1]{}
\newcommand{\response}[1]{}
\newcommand{\deleting}[1]{}

\newif\ifreport\reporttrue

\ifreport
\else
\fi

\begin{document}

\title{Asynchronous Fractional Multi-Agent Deep Reinforcement Learning for \\ Age-Minimal  Mobile Edge Computing}

\author{
    Lyudong Jin,
    Ming Tang,  ~\IEEEmembership{Member,~IEEE},
    Jiayu Pan,  ~\IEEEmembership{Member,~IEEE},
    Meng Zhang,  ~\IEEEmembership{Member,~IEEE},
    Hao Wang,  ~\IEEEmembership{Member,~IEEE}
    \thanks{
    This work is supported in part by the National Natural Science Foundation of China under grant Nos. 62572433, and in part by the National Natural Science Foundation of China under grant No. 62402435,
by Ningbo Yongjiang Talent Programme 2023A-398-G, and by Natural Science Foundation of Ningbo 2024J205. (Corresponding author: Meng Zhang.)
    
    Lyudong Jin and Meng Zhang are with the Zhejiang University—University of Illinois at Urbana–Champaign Institute, Zhejiang University, Haining 314400, China (e-mail: 3180101183@zju.edu.cn; mengzhang@intl.zju.edu.cn). 
    
      Ming Tang is with the Department of Computer Science and Engineering, Southern University of Science and Technology, Shenzhen 518000, China (e-mail: tangm3@sustech.edu.cn).

      J. Pan is with the School of Software Technology, Zhejiang University,
Hangzhou, China (e-mail: jiayupan26@zju.edu.cn).
    
    Hao Wang is with the Department of Data Science
    and AI, Faculty of Information Technology, Monash University, Melbourne, Victoria 3800, Australia (e-mail: hao.wang2@monash.edu).}
    \thanks{Part of this paper was presented in AAAI-24 \cite{jin2024fractional}.}
}



\maketitle

\begin{abstract}
\edit{In the realm of emerging real-time networked applications such as cyber-physical systems (CPS), the \textit{Age of Information (AoI)} has emerged as a pivotal metric for evaluating timeliness.}
\rev{To meet the high computational demands, such as those in smart manufacturing within CPS, mobile edge computing (MEC) presents a promising solution for optimizing computing and reducing AoI.}
\rev{In this work, we study the timeliness of compute-intensive updates and explore the joint optimization of task-updating and offloading policies to minimize AoI. }
We consider edge load dynamics and formulate a task scheduling problem to minimize the expected time-average AoI.
\rev{Solving this problem is challenging due to the fractional objective introduced by AoI and asynchronous decision-making in the semi-Markov game (SMG).
 }
To this end, we propose a fractional reinforcement learning (RL) framework. We begin by introducing a fractional single-agent RL framework and establish its linear convergence rate. \revb{Building on this, we develop a fractional multi-agent RL framework, extend Dinkelbach's method, and demonstrate its equivalence to the inexact Newton's method.} \revbb{Furthermore, we provide the conditions under which the framework achieves local linear convergence to the Nash equilibrium (NE). }
To tackle the challenge of  asynchronous decision-making in the SMG, we further design an asynchronous 
 model-free fractional \rev{multi-agent RL} algorithm, where each mobile device can determine the task updating and offloading decisions  without knowing the real-time system dynamics and decisions of other devices.
\revb{Compared with the best existing baseline, the proposed algorithm reduces the average AoI by up to $42.0\%$.
} 
\end{abstract}


\section{Introduction}

\subsection{Background and Motivations}
Real-time embedded systems and mobile devices, including smartphones, IoT devices, and wireless sensors, have rapidly proliferated. This growth is driven by cyber-physical systems (CPS) applications  including collaborative autonomous driving for autonomous logistics (e.g., \cite{katare2023survey}), smart energy systems (e.g., \cite{xu2023power}), and smart manufacturing (e.g., \cite{jin2023cloud}). \revbb{These applications generate massive data and require intensive computational processing, which in turn creates a pressing need for low-latency, reliable, and private task execution.}  Mobile edge computing (MEC), also known as multi-access edge computing \cite{porambage2018survey}, addresses these needs by offloading computational tasks from end devices to nearby edge servers  \cite{mao2017survey}. \revb{Thus, MEC reduces latency compared to cloud computing and provides larger computational capacity than that of mobile computing.}

\revbb{Timely information updates are crucial for the emerging  CPS applications with high computational demands.} These updates depend on real-time sensory data and computational results, such as human interaction and safety data, environmental data, and positioning and motion data.
For instance, critical applications (e.g., \cite{zhu2018vehicular,xu2020vehicular}), such as collaborative robots in smart manufacturing that operate alongside humans, require split-second decision-making to ensure safety and efficiency, relying on computationally intensive tasks like object detection, path planning, and task execution \cite{el2019cobot}.
However, local computing often struggles to meet these stringent timeliness requirements, potentially compromising system performance and safety. This need for data freshness has driven the development of a new metric, \textit{Age of Information (AoI)} \cite{kaul2012real}, which quantifies the time elapsed since the most recently delivered data or computational results were produced, providing an accurate measure of information timeliness.

\IEEEpubidadjcol

\rev{While numerous prior studies on MEC have focused on minimizing delay (e.g., \cite{wang2021delay,Tang2020TMC}), they often overlooked AoI, which is crucial for real-time applications.  
Here we highlight the huge difference between delay and AoI. Task delay measures the duration between task generation and task output reception. \revb{With less frequent updates (i.e.,
when tasks are generated at a lower frequency), task delays are naturally smaller due to reduced queuing delays from empty queues.}   In contrast, AoI considers both task delay and the output freshness. 
To minimize AoI for compute-intensive tasks, the update frequency must be balanced to reduce individual task delays while ensuring the freshness of the most recent task output.  This difference between delay and AoI reveals a counterintuitive phenomenon in age-minimal scheduling: mobile devices may need to wait before generating new tasks after receiving a task output.  }
\edit{A key bottleneck in CPS is the requirement for timely operation. Although low-latency communication is necessary, research has shown that it alone cannot guarantee timely system performance (see \cite{AoIsurvey,Talak2021TIT}). 
}

\revbb{Hence, AoI is a more suitable metric than latency for timely updates in CPS systems.}

In this paper, we aim to answer the key question as follows.
\begin{kquestion}
\revr{How should mobile devices optimize their updating and offloading policies in real-time MEC systems to minimize  AoI?}
\end{kquestion}

%

\subsection{Challenges}

Optimizing AoI in MEC systems with time-varying channels requires a sophisticated  scheduling policy for mobile devices, involving two key decisions: updating and offloading.  The \textit{updating} decision determines the  interval between task completion and the initiation of the next task generation, while the \textit{offloading} decision decides whether to process the task locally or offload it to a specific edge server. Given the success of multi-agent reinforcement learning (MARL)  in MEC systems \cite{hao2023exploration}, we aim to design a  fractional MARL framework with theoretical analysis and propose an asynchronous fractional multi-agent deep reinforcement learning (MADRL)   to optimize the task scheduling policy. \edit{The main challenges are as follows.}

    \textit{How to deal with the fractional nature of the AoI objective in RL?}  Many prior works on MEC have explored offloading  (e.g., \cite{Tang2020TMC,ma2022green,zhao2022multi,he2022age,chen2022info}),  and some have considered \rev{AoI (e.g.,} \cite{xu2022aoi,he2022age,chen2022info}). \edit{However,  they did not consider the  fractional AoI objective.
    The fractional RL framework focuses on optimizing a fraction-of-expectation objective, making it challenging to directly evaluate the impact of actions on the objective, unlike the reward-based approach in the standard RL framework.
Real-time MEC systems require a fractional RL framework to address the fractional AoI objective.} 

    \revb{\textit{How can we adapt the fractional framework to a multi-agent scenario?}} \revb{Many MEC systems involve a decision-making process among multiple mobile devices.}  Effectively addressing this multi-agent paradigm in MARL  is important for task scheduling in multi-agent systems.  \edit{However, the fractional framework for a single agent cannot be directly adapted to the multi-agent scenario due to the increased complexity and inter-dependencies among agents.} 
    Thus, the MARL framework faces additional challenges in designing the fractional framework for the AoI objective. 
    
    \textit{How to address the asynchronous decision problem in real-time MEC systems?}  Traditional MADRL algorithms, such as multi-agent deep deterministic policy gradient (MADDPG) \cite{lowe2017multi}, assume that agents make decisions synchronously. \rev{However, in multi-agent MEC systems, each agent (mobile device) has different task update and processing durations,
 leading to asynchronous decision-making among agents.
    \revbb{Traditional MADRL algorithms including QMIX \cite{rashid2020monotonic} and MAPPO \cite{yu2022surprising} would be inefficient for  asynchronous control  in the MEC scheduling problem, motivating new approaches that can explicitly handle asynchronous agent interactions. } }

\subsection{Solution Approaches and Contributions}


\revbb{ In this paper, we address AoI minimization in MEC through a unified framework for joint task updating and offloading. The novelty lies in three connected aspects: a new asynchronous SMG formulation, a fractional multi-agent learning framework with convergence analysis, and an asynchronous algorithmic design. This setting cannot be addressed by directly applying delay-oriented MEC methods or standard synchronous MARL algorithms.}

Our main contributions are summarized as follows:

\begin{itemize}
    \item \revbb{ We formulate an age-minimal MEC problem with joint task updating and offloading, where multiple mobile devices interact through an asynchronous SMG. This formulation captures both the fractional nature of the AoI objective and the asynchronous decision process induced by heterogeneous task updating and completion times. To the best of our knowledge, this is the first work to develop multi-agent asynchronous policies for joint updating and offloading optimization in age-minimal MEC.}

     \item \revbb{We develop a fractional RL/MARL framework for the AoI ratio objective. In the single-agent case, we combine RL with Dinkelbach's method and prove linear convergence. In the multi-agent case, we extend the framework to Markov games, establish the existence of a Nash equilibrium, interpret the outer-loop update as an inexact Newton step, and derive sufficient conditions for local linear convergence. This provides a principled alternative to directly applying existing MARL methods to MEC and offers a new framework for fractional AoI optimization in multi-agent systems.}

   \item \revbb{We propose an asynchronous fractional MADRL algorithm under the CTDE paradigm. The algorithm integrates asynchronous trajectory collection with GRU-based history aggregation to handle asynchronous interactions and hybrid  actions, while enabling decentralized execution using only local information. The asynchronous design is therefore directly motivated by the MEC scheduling dynamics considered in this paper.}

     \item \revbb{Extensive experiments validate the proposed framework and show that the proposed algorithm significantly outperforms representative baselines, reducing the average AoI by up to $42.0\%$.}
\end{itemize}

\section{Literature Review}\label{sec:review}

\textbf{Mobile Edge Computing}:
\rev{Existing  MEC research has explored various areas, including resource allocation  \cite{10.1145/3597023}, service placement \cite{taka2022service}, proactive caching \cite{liu2022distributed}, and task offloading \cite{wang2022decentralized}.
Many studies have  proposed RL-based approaches to optimize  task delays in a centralized manner (e.g., \cite{huang2020deep,tuli2022dynamic}) or in a decentralized manner (e.g., \cite{Tang2020TMC,liu2022deep}).}  
\rev{ \emph{Despite the success  in reducing the task delay, these approaches are NOT easily applicable to age-minimal MEC due to the fractional objective and asynchronous decision-making.} }

\textbf{Multi-Agent RL in MEC:} Multi-Agent RL has been widely adopted in multi-agent approaches in a decentralized manner in MEC systems. Li \textit{et al.} proposed a decentralized edge server grouping algorithm and achieved NE by proving it to be an exact potential game \cite{li2021user}. Chen \textit{et al.} addressed the task offloading in the information freshness-aware air-ground integrated multi-access edge computing by letting each non-cooperative mobile user behave independently with local conjectures utilizing double deep Q-network. Feng \textit{et al.} utilized a gating threshold to \rev{intelligently} choose between local and global observations and limit the information transmission for approximating  Nash equilibrium in an anti-jamming Markov game \cite{FENG2023330}. \rev{However, these MARL approaches have not considered the fractional objective when reaching NE.}

\textbf{Age of Information:} \revb{ Since its introduction by Kaul \textit{et al.}  \cite{kaul2012real}, AoI has attracted increasing attention in the study of CPS (e.g., \cite{yuan2023aoi,fu2023aoi}) in recent years. However, the majority of works have focused on the optimization of AoI in queueing systems and wireless networks, assuming the availability of  complete and known statistical information (see \cite{AoISurvey2021,AoIMEC1}).}

  \revb{Additionally, recent works have characterized the optimal sampling policies in single-source MEC systems \cite{zhu2024optimizing} and further extended the analysis to joint sampling and scheduling optimization in multi-source systems \cite{zhu2025multi}.
  However, the above studies analyzed either single-device or single-server scenarios and are difficult to adapt to the complex, time-varying dynamics of multi-user MEC systems where channel states and edge loads fluctuate unpredictably.  }
A few studies have investigated RL algorithm designs to minimize AoI in various application scenarios, including wireless networks \cite{ceran2021reinforcement},  Internet-of-Things  \cite{AoIDeep2},  autonomous driving \cite{xu2022aoi}, vehicular networks  \cite{AoIDeep1}, and UAV-aided networks  \cite{UAVDeep2}. 
They mainly focused on \rev{optimizing} resource allocation and trajectory design.
Some works  
considered AoI as the performance metric for task offloading  in MEC and proposed RL-based approaches. 
For instance, Chen \textit{et al.} in \cite{chen2022info} considered AoI to capture the freshness of computation outcomes and proposed a multi-agent RL algorithm. 
\emph{However, these works focused solely on designing task offloading policies without jointly optimizing updating policies.}
\textit{\revb{Note that none of these approaches have considered fractional RL and are difficult to directly address the problem  in this paper.}}

\textbf{RL with Fractional Objectives:} \revb{There is limited research on RL with fractional objectives.}  Ren \textit{et al.} introduced fractional MDP \cite{1431043}. Tanaka \textit{et al.} \cite{doi:10.1080/02522667.2015.1105525} further studied partially observed MDPs with fractional costs. However, these studies have not been extended to an RL framework. \rev{Suttle \textit{et al.} \cite{pmlr-v139-suttle21a} proposed a two-timescale RL algorithm for fractional cost optimization. However, this method requires fixed reference states in Q-learning updates, which cannot be directly adapted to an asynchronous multi-agent environment.}

\textbf{Asynchronous Decision-Making MARL:} \revb{Research on asynchronous decision-making in MARL is also limited. Xiao \textit{et al.} \cite{xiao2022asynchronous} established an independent actor with individual centralized critic framework collecting historical information of each agent. But their approach disregards historical information from other agents and demands frequent policy updates via on-policy training.} \revb{Chen et al. \cite{chen2021VarlenMARL} proposed VarlenMARL in which each agent  has a variable time step to gather the padding data of the most up-to-date step of other agents. However, this approach becomes inefficient with a large number of agents and may neglect information from agents that make decisions frequently.} \edit{Wang and Sun \cite{wang2021reducing} proposed a credit assignment framework with graph-based event critic for bus route control. \revb{However, this method only considers the effect of nearest upstream and downstream events of a single line and is not suitable for more complicated patterns including the effects from events of multiple lines.} MAC-IAICC \cite{xiao2022asynchronous} formulates asynchronous actions as macro-actions and rigorously models the task within the MacDec-POMDP framework.    Liang \textit{et al.}  addressed these issues with  ASM-PPO\cite{liang2022asm} and ASM-HPPO\cite{liang2023asynchronous}, which collected agent trajectories independently and trained using available agents at each time slot. However, these methods could still suffer from the disruption of the synchronization of agents' tasks, as we will study later.  }



\section{System Model}\label{sec:model}

In this section, we \rev{present the system model for MEC.} \edit{We first \rev{introduce} the device and edge server model.}
We then define the AoI in our MEC setting and formulate the task scheduling problem within an SMG framework.

In this MEC system model, we consider mobile devices $\M=\{1,2,\cdots, M\}$ and edge servers $\N=\{1,2,\cdots, N\}$, as shown in Fig.~\ref{fig:model}. \rev{We denote the set of tasks as $\K=\{1, \cdots, K\}$.}
\begin{figure}
    \centering
    \includegraphics[width=8cm]{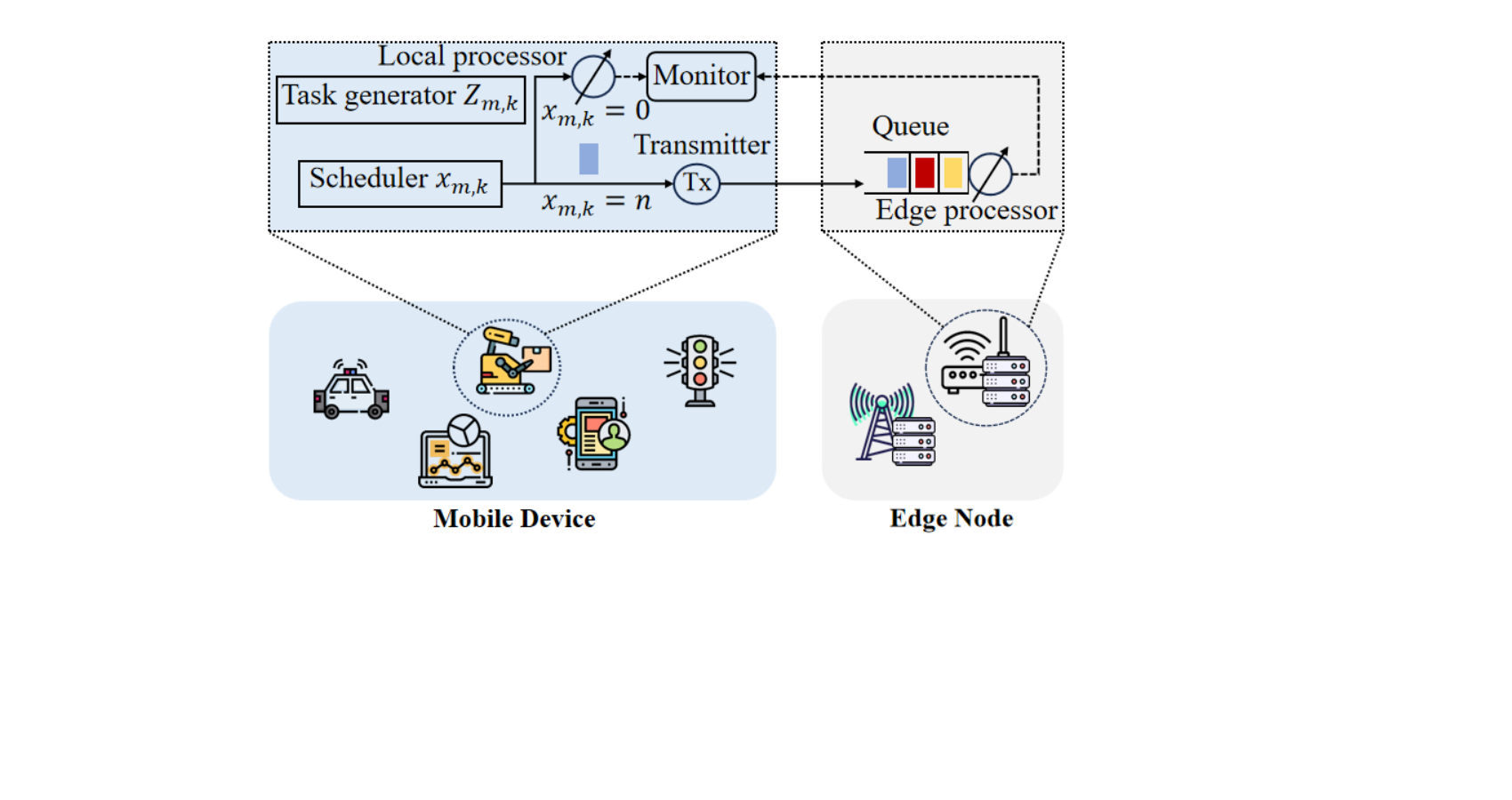}
    \caption{An illustration of an MEC system with a mobile device $m\in\M$ and an edge server $n\in\N$ where the tasks offloaded by different mobile devices are represented using different colors.
    }
    \vspace{-10pt}
    \label{fig:model}
\end{figure}
\subsection{Device Model}
\edit{As shown in Fig. \ref{fig:model}, \rev{each} mobile device \rev{$m$} is associated with a mobile user of CPS applications (e.g., \cite{katare2023survey,xu2023power,jin2023cloud,porambage2018survey,mao2017survey,zhu2018vehicular,xu2020vehicular,el2019cobot}), 
who seeks to offload their computationally intensive tasks to edge servers to receive timely updates.}
Each mobile device
includes a task generator, a scheduler,  \rev{a monitor, and a local processor}. \rev{The generator produces computational tasks. The scheduler chooses to process them locally or offload to an edge server}. After the edge server or local processor processes the task, the result is sent to the monitor and the generator then decides when to initiate a new task\deleting{, which we refer to as an \textit{update}}. 

\textbf{Generator}: \edit{We adopt a \textit{generate-at-will} model, as considered in \cite{sun2017tit,arafa2019timely,sun2019sampling,arafa2019age,Zhang21JSAC}, where each mobile device's task generator decides when to generate a new task. For example, tasks are generated alongside sensory data, with sensors specifically designed to sample physical phenomena on demand.}\footnote{\edit{For the alternative model to the generate-at-will paradigm, where tasks arrive randomly, extending the proposed fractional RL framework is conceptually straightforward, as this can be achieved by simply removing the waiting action.}} Specifically, \rev{when task $k-1$ of mobile device $m\in\M$  is completed at $ t'_{m,k-1}$, the task generator observes the latency $Y_{m,k-1}$ of task $k-1$ and queue information of edge servers  $\q(t'_{m,k-1})$. 

The generator then decides on $Z_{m,k}\geq 0$, i.e., the waiting time before generating the \emph{next} task $k$. We assume that edge servers share their load levels in response to mobile device queries. Since a generator produces a new task after the previous task is finished, the queue length is less than or equal to $M$, which incurs minimal signaling overhead. \revbb{ Let $a^{\textsc{U}}_{m,k}$ denote the updating action for task $k$, i.e., $a_{m,k}^{\textsc{U}}=Z_{m,k}$} and $\mathcal{A}^{\textsc{U}}=\mathbb{R}_+^{M}$ be the updating action space, where $\mathbb{R}_{\geq 0}$ denotes the space of nonnegative real numbers.}
\rev{The generation time of  task $k$ is calculated as $t_{m,k} = t'_{m,k-1} + Z_{m,k}$.}
  Notably, an optimal waiting strategy for updating can outperform a zero-wait policy, indicating that the \rev{waiting} time $Z_{m,k}$ is not necessarily zero and should be optimized for efficiency  \cite{sun2017tit}. \revb{ Specifically, our framework optimizes $Z_{m,k}$ to find the optimal balance between the waiting time $Z_{m,k}$ and potential queuing delays at the server, which is  crucial in multi-agent environments where server availability is stochastic and transient. }

\textbf{Scheduler}: 
When the generator produces task $k$ at time $t_{m,k}$, the scheduler makes the offloading decision $x_{m,k}\in\{0\}\cup\N $ with the queue information $\q(t_{m,k})$ of edge servers\footnote{Under the system model in Section \ref{sec:model}, observing the queue lengths of edge servers is sufficient for mobile devices to learn their offloading policies. Under a more complicated system (e.g., with device mobility), additional state information may be necessary. However, our proposed  RL-based framework remains applicable with an extended state vector.}.       \revbb{At execution time, each mobile device makes updating and offloading decisions using only locally available information, including its task and latency state and the queried edge-server queue information. No global state, other agents’ real-time actions, or centralized history broadcasts are assumed to be available online.
}
 Let $a_{m,k}^{\textsc{O}}$ denote the action of task $k$ from mobile device $m$. \rev{That is}, $a_{m,k}^{\textsc{O}} = x_{m,k}, k\in\K, m\in\M$. Let $\mathcal{A}^{\textsc{O}} \in(\{0\}\cup\N)^M$ and $\pio_m$ denote the offloading action space and the task offloading policy of mobile device $m\in\M$, respectively.
For local processing on mobile device $m\in\M$,  $\tau^{\text{local}}_{m,k}$ (in seconds) represents the processing time  of task $k$. 
The value of $\tau^{\text{local}}_{m,k}$ depends on the task size and the real-time computational capacity of the mobile device (e.g., whether the device is busy in processing tasks of other applications).
For  offloading task $k$ to edge server $n\in\N$, 
$\tau^{\text{tran}}_{n,m,k}$ (in seconds) denotes the \rev{transmission latency of mobile device $m$}. 
The value of $\tau^{\text{tran}}_{n,m,k}$ depends on the time-varying wireless channels.
\revb{We assume that   $\tau^{\text{local}}_{m,k}$ is a random variable that follows an exponential distribution, as in \cite{tang2021age,liu2022deep,zhu2022online},  to capture the inherent  variability caused by stochastic computational complexity and resource contention in practical MEC environments.}

\subsection{Edge Server Model}

When a mobile device $m\in\mathcal{M}$  offloads a task to edge server $n\in\N$, we consider the time-varying channels between  mobile devices and edge nodes located at different positions.  Considering Rayleigh fading, the instantaneous channel power gain between mobile user $m$ and the edge server $n$ at global time $t$ is given by $h_{t,m,n} = g_{t,m,n}d_{t,m,n}^{-\alpha}$, where $d_{t,m,n}$ is the distance from mobile user $m$ to the edge node $n$, and $\alpha$ is the path loss exponent, and $g_{t,m,n}$ is the Rayleigh fading coefficient. Simultaneous transmissions from multiple devices to the same edge node can cause interference.    According to the Shannon principle, the channel capacity for mobile device $m$ with $N$ orthogonal sub-channels of edge nodes can be derived as:
\begin{align}
    \revb{R_{t,m,n} = W \log_2 \left( 1 + \frac{P_m h_{t,m,n}}{\eta_0 + \sum_{i \in \mathcal{M} \setminus \{m\} : c_i = c_m} \beta_{t,i} P_i h_{t,i,n}} \right),}
\end{align}
\revb{where $W$ is the channel bandwidth, $P_m$ denotes the transmit power of mobile user $m$, and $\eta_0$ represents the background noise power. The variable $\beta_{t,i} \in \{0,1\}$ indicates the transmission status of user $i$ at time $t$ ($\beta_{t,i}=1$ for active transmission and $0$ otherwise). Note that the interference term implies that only mobile devices transmitting concurrently on the same sub-channel $c_m$ contribute to the interference.} Then, the transmitting time can be calculated as $\tau_{n,m,k}^{\text{tran}}=l^kd^k/R_{t,m,n}$.   

\rev{When the mobile device $m$ successfully offloads the task to edge server $n$,} the task is stored in a queue waiting for processing, as shown in Fig.~\ref{fig:model}. We assume the queue operates on a \rev{first-come-first-served (FCFS)} basis \cite{AoISurvey2021}. \revb{FCFS prevents strategic preemption to ensure game stability, while our generate-at-will setting inherently eliminates self-blocking.} \rev{We denote $\q(t) = \{q_n(t)\}_{n\in\N}$ as the queue lengths of all edge servers. Edge servers update their queue length information  $\q(t) $ in two cases: when a task is completed or when a waiting period ends. Because a generator produces a new task only after the previous task has been processed, the queue length is at most $M$. Thus, the information to be sent can be encoded in  $O (\log_2 M)$ bits with minimal signaling overhead.} 
\rev{We denote $w^{\text{edge}}_{n,m,k}$ (in seconds) as the duration that task $k$ of mobile device $m\in\M$ waits at the queue of edge server $n$.} Let $\tau^{\text{edge}}_{n,m,k}$ (in seconds) denote the latency of edge server $n$ for processing task $k$ of mobile device $m$. 
We assume that $\tau^{\text{edge}}_{n,m,k}$ is a random variable  following an exponential distribution, i.e., $\tau^{\text{edge}}_{n,m,k}\sim Exp(\frac{C^{\text{edge}}_{n}}{l^k d^k})$. Here, $C^{\text{edge}}_{n}$ is the processing capacity (in GHz) of edge $n$. \revb{Parameters $l^k$ and $d^k$ denote the task size and task density of task $k$, respectively. } 
In addition, the value of $w^{\text{edge}}_{n,m,k}$ depends on the processing times of the tasks  ahead in the queue, which are possibly offloaded by other  mobile devices. \rev{Additionally, this offloading information cannot be directly observed by mobile device $m$.}  \revb{Therefore, the waiting time cannot be modeled analytically, motivating our use of a  model-free MARL framework to capture such dynamics. }

\subsection{Age of Information}

For mobile device $m$, the AoI at global clock $t$ \cite{AoISurvey2021} is defined by 
\begin{align}
    \Delta_m(t)= t- T_{m}(t), ~~\forall m\in\mathcal{M}, t\geq 0,
\end{align}
where $T_{m}(t)\triangleq \mathop{\rm{\max}}_{k} (t_{m,k}| t'_{m,k}\leq t)$ stands for the time stamp of the most recently completed task.

The overall duration to complete task $k$ is denoted by $Y_{m,k}\triangleq t'_{m,k}-t_{m,k}$.
Therefore, 
\begin{align}
  \hspace{-0.35cm}  Y_{m,k}=\begin{cases}
     \tau^{\text{local}}_{m,k},&x_{m,k}=0,\\
     \tau^{\text{tran}}_{n,m,k} + w^{\text{edge}}_{n,m,k} + \tau^{\text{edge}}_{n,m,k},&x_{m,k}=n\in\mathcal{N}.
    \end{cases}
\end{align}
\rev{We consider a deadline $\bar{Y}$ (in seconds).} Tasks not completed  within $\bar{Y}$ seconds will be dropped \cite{Tang2020TMC,li2020age}. Meanwhile, the AoI keeps increasing until the next task is completed.

We define the trapezoid area associated with time interval $[t_{m,k},t_{m,k+1}]$: 
\begin{align}
    &A(Y_{m,k},Z_{m,k+1},Y_{m,k+1})\nonumber\\
    \triangleq& \frac{1}{2}(Y_{m,k}+Z_{m,k+1}+Y_{m,k+1})^2-\frac{1}{2}(Y_{m,k+1})^2, \label{trapezoid}
\end{align}
where $Z_{m,k+1}$ denotes the updating interval
before generating the next task $k+1$. \revb{The time-average AoI is defined as $\lim_{T \to \infty} \frac{1}{T} \int_0^T \Delta_m(t) dt$. By applying the Renewal Reward Theorem, this time-average equals the ratio of the expected area per cycle (given by \eqref{trapezoid}) to the expected cycle duration. Thus, the objective for mobile device $m$ is characterized as:}
\begin{align}
    {\Delta}^{(ave)}_m\triangleq&\liminf_{K\rightarrow \infty}\frac{\sum_{k=0}^K  \mathbb{E}[A(Y_{m,k},Z_{m,k+1},Y_{m,k+1})]}{\sum_{k=0}^K \mathbb{E}[Y_{m,k}+Z_{m,k+1}]},\label{aoi}
\end{align}
where $\mathbb{E}[\cdot]$ is the expectation with respect to decisions made according to certain policies, which will be elaborated upon in the subsequent sections.
\subsection{Game Formulation}
\rev{In MEC systems, mobile devices make updating and offloading decisions asynchronously due to variable transition times. These transition times fluctuate based on scheduling strategies and edge workloads. \rev{Given that the current MARL framework is designed for synchronous decision-making processes, it is not suitable for application in this context.} To address this challenge, we model the real-time MEC scheduling problem as an SMG. An SMG extends Markov decision processes to multi-agent systems with variable state transition times. Thus, this framework can address asynchronous decision-making in MEC systems.}

\edit{The game is defined as $(\M, \mathcal{S},\mathcal{A},P_{\mathcal{S}}, P_F, \mu_0)$, given by}
\begin{itemize}
    \item $\M$: the set of  mobile devices;
   \edit{\item $\mathcal{S}$: the state space denoted by $\mathbb{N}^{M}\times\mathbb{N}^{NM}\times\mathbb{R}^M$.}  Specifically,
    a state $\bs\in\mathcal{S}$ is expressed as $\bs\triangleq((I_m)_{m\in\M},(\q(t_m))_{m\in\M}, (Y_{m})_{m\in\M})$, where $(I_m)_{m\in\M}$ denotes the decision indicators, specifying whether  each mobile device $m\in\M$ needs to take offloading action $a^{\textsc{O}}_m$, updating action $a^{\textsc{U}}_m$ or make no decisions;
    \item  $\mathcal{A}$: the action space defined as $\mathcal{A}^{\textsc{U}}\times\mathcal{A}^{\textsc{O}}$. Specifically, an action  $\ba\in\mathcal{A}$ consists of its task updating action $\ba^{\textsc{U}}$ and task offloading action $\ba^{\textsc{O}}$.  We define the set of admissible state-action pairs as $\H=\{(\bs, \ba)|\bs\in\mathcal{S},\ba\in\mathcal{A}\}$. 
    \item  $P_{\mathcal{S}}:\H\rightarrow \mathbb{P}(\mathcal{S})$ is the transition function of the game, where $\mathbb{P}(\mathcal{S})$ is the collection  of probability distributions over \rev{space  $\mathcal{S}$}.
    \item $P_F: \H\times\mathcal{S}\rightarrow \mathbb{P}(\mathbb{R}^{+})$: the stochastic kernel determining the transition time distribution;
    \item $\mu_0$: the  distribution of the initial state $\bs_0$.
\end{itemize}

\edit{The key difference between an SMG and a conventional Markov game lies in the timing of state transitions. In a traditional Markov game, state transitions occur at fixed, discrete intervals, with agents acting synchronously at each step. In contrast, an SMG features asynchronous transitions, where the time between states is a non-negative random variable drawn from a distribution dependent on the current state and actions. Thus, the decision-making of each agent is asynchronous, with each agent acting independently based on its own timeline.}




\rev{We define the policy by  $\pii=\{\pii_m\}_{m\in\M}$, where each policy} $\pii_m=(\pii_m^{\textsc{U}},\pii_m^{\textsc{O}})$  includes both the updating and offloading policies of single mobile device $m$. 
We define the expected long-term discounted AoI of mobile device $m$ as 
\begin{align}
\mathbb{E}_{\pii}[\Delta_m] 
\triangleq \mathop{\mathbb{E}}\limits_{\bs_0\sim\mu_0}\left\{ \mathop{\frac{ \sum\limits_{k=0}^{\infty}\delta^{k} \mathbb{E}_{\pii}[A(Y_{m,k},Z_{m,k+1},Y_{m,k+1})]}{\sum\limits_{k=0}^{\infty}\delta^{k} \mathbb{E}_{\pii}[\revr{Y_{m,k}}+Z_{m,k+1}]}}\right\},
\label{discount-aoi}
\end{align}
 \rev{where $\delta\in(0,1)$ is a discount factor. 
 We have the objective with the initial state $\bs_0$ under initial distribution $\mu_0$. For simplicity, we omit this notation, which will be understood by default unless otherwise stated in the following.
 } 

 \edit{While standard RL is typically designed to maximize or minimize expected discounted cumulative rewards, the fractional RL framework focuses on optimizing a fraction-of-expectation objective function, as in \eqref{discount-aoi}.
This means that it is challenging to directly evaluate the impact of a specific action on the objective value, as is done with rewards in the conventional RL setting. This challenge primarily motivates our fractional RL framework, which also necessitates new analyses for convergence and algorithm design.}
 
 Note that the discounted AoI form facilitates the application of AoI optimization in  RL and MARL  frameworks that utilize   discounted cost functions. \revb{By assigning less weight to costs incurred in the distant future, the objective in \eqref{discount-aoi} encourages the agent to prioritize immediate information freshness in non-stationary environments, where long-term planning faces high uncertainty.} Additionally, as  $\delta$ approaches 1, the discounted AoI  \eqref{discount-aoi}  approximates the undiscounted AoI function \eqref{aoi}. \rev{The expectation $\mathbb{E}_{\pii}[\cdot]$ is taken with respect to the policy $\pii$ and the time-varying system parameters, e.g.,  channel states, processing times and the edge loads.} We define the Nash equilibrium $\pii^{\star}=\{\pii_m^{\star}\}_{m\in\M}$ for our task scheduling problem  as follows.

\begin{definition}[Nash Equilibrium] \label{NE}
     In the stochastic game $\bG$, a Nash equilibrium is a tuple of  \rev{policies} $\left(\pii_1^{\star},\dots,\pii_M^{\star}\right)$ such that for all  $m\in\M$ we have,
\begin{equation}
\mathbb{E}_{\pii^{\star}_m, \pii^{\star}_{-m}}[\Delta_m] \leq \mathbb{E}_{\pii_m, \pii^{\star}_{-m}}[\Delta_m].
\end{equation}
\end{definition}

  For each mobile device $m\in\M$, given the fixed stationary policies $\pii_{-m}^{\star}$,  the best response  in our real-time MEC scheduling problem is defined as
 \begin{align}
      \pii^{{\star}}_m = \mathop{\arg \mathop{\rm{minimize}}}\limits_{\pii_m} ~  \mathbb{E}_{\pii_m, \pii^{\star}_{-m}}[\Delta_m].\label{problem}
 \end{align}

\revbb{Equation \eqref{discount-aoi} specifies a per-device objective. Accordingly, the MEC scheduling problem is formulated as a decentralized game with coupled AoI objectives, rather than a centralized network-wide AoI minimization problem. Since each device acts on local information and does not observe other devices’ real-time actions online, Nash equilibrium is adopted to characterize a stable policy profile.} \rev{Solving \eqref{problem} is challenging for three reasons. First, the fractional objective introduces a major challenge in obtaining the optimal policy  for conventional RL  algorithms due to its non-linear nature. Second, the multi-agent environment introduces complex dynamics from interactions between edge servers and mobile devices.  Third, the continuous-time nature, stochastic state transitions, and asynchronous decision-making in the SMG pose significant challenges in directly applying conventional game-theoretic methodologies.}

In the following sections, we first analyze a simplified version of the single-agent fractional RL framework. This framework establishes a robust theoretical foundation and serves as a benchmark for subsequent experiments. \revb{We then extend our theoretical analysis of our fractional framework to multi-agent scenarios, which is closer to SMG and practical dynamics.} Finally, we propose the fractional MADRL algorithm in an asynchronous setting for the SMG, which matches the practical dynamics of real-time MEC systems. 

\section{ Fractional Single-Agent RL Framework}\label{sec:solution}
\revb{For foundational  theoretical analysis, we first consider a fractional single-agent RL framework. We design a fractional MDP and propose a fractional Q-learning algorithm.} We provide  analysis of the linear convergence for this algorithm.  We formulate the problem as follows.
\subsection{Single-Agent Problem Formulation}

 \edit{ Recall that agent $m$ selects its updating action and scheduling action, satisfying $(a_{m,k}^O, a_{m,k}^U) \in \mathcal{A}_m \triangleq \mathbb{R}^+ \times (\{0\} \cup \mathcal{N})$.} The goal of single-agent RL is to learn a stationary policy $\pi_m$ that maps the state space $\mathcal{S}$ to agent $m$'s action space $\mathcal{A}_m$ by solving the following single-agent problem:
 \begin{align}
      \pii_m^{*} = \mathop{\arg \mathop{\rm{minimize}}}\limits_{\pii_m} ~  \mathbb{E}_{\pii_m}[\Delta_m],\label{single-agent}
 \end{align}
 \rev{where we focus on the policy of  agent $m$ alone, compared with \eqref{problem}.} 

 We study the general fractional MDP framework and drop the index $m$ in the rest of this section. 
 We then introduce a fractional RL framework to address Problem \eqref{single-agent}. \rev{Specifically, we first} reformulate Problem \eqref{single-agent}.  We then introduce a fractional RL framework and propose a fractional Q-learning algorithm with provable convergence guarantees.

%

\subsection{Dinkelbach's Reformulation}
\revr{To solve Problem \eqref{single-agent}, we consider the Dinkelbach's reformulation.}
Specifically, we define a reformulated AoI in a discounted fashion:
\begin{align}
\mathbb{E}_{\pii}[\Delta', \gamma] 
\triangleq &\sum\limits_{k=0}^{\infty}\delta^{k} \mathbb{E}_{\pii} [   A(Y_k,Z_{k+1},Y_{k+1})\nonumber \\
&- \gamma(Y_k+Z_{k+1})], \forall \gamma\geq0.\label{dinkel-aoi}
\end{align}

 Let $\gamma^*$ be the optimal value of  Problem \eqref{single-agent}. By leveraging Dinkelbach's method \cite{dinkelbach67ms}, we  reformulate the problem as follows.
 \begin{lemma}[Asymptotic Equivalence\cite{dinkelbach67ms}]\label{L1}
 Problem \eqref{single-agent} is equivalent to the following reformulated problem:
 \begin{align}
  \pii^*  = \arg\mathop{\rm{minimize}}_{\pii}~ \mathbb{E}_{\pii}[\Delta',\gamma^*],\label{eq:average-p}
 \end{align}
 where $\pii^*$ is the optimal solution to Problem \eqref{single-agent}.
 \end{lemma}
 
\rev{Since $\mathbb{E}_{\pii}[\Delta',\gamma^*]\geq 0$ for any $\pi$ and $\mathbb{E}_{\pii^*}[\Delta',\gamma^*]= 0$, $\pii^*$ is also optimal for the Dinkelbach reformulation. This implies that the reformulation equivalence is also established for our stationary policy space.}

\subsection{Fractional MDP}
\revb{We then extend the standard MDP to a fractional form, laying the groundwork for a more complex and practical setting in a Markov game.}

A fractional MDP is defined as $(\mathcal{S},\mathcal{A},c_N,c_D,Pr,\delta,\mu_0)$, where $\mathcal{S}$ and $\mathcal{A}$ are the finite sets of states and actions, respectively; $Pr$ is the state transition distribution; $c_N$ and $c_D$ are the fractional costs,
 $\delta$ is a discount factor, and $\mu_0=\{\mu_0(\bs)\}_{\bs\in\mathcal{S}}$ denotes the initial  global state distribution. We use $\mathcal{Z}$ to denote the joint state-action space, i.e., $\mathcal{Z}\triangleq\mathcal{S}\times \mathcal{A}$.
We define the instant fractional costs at task $k$  as
\begin{align}
&c_{N}(\bs_k, \ba_k) =A(Y_{k},Z_{k+1},Y_{k+1}), \\
&c_{D}(\bs_k, \ba_k) = Y_{k}+Z_{k+1}.
\end{align}
From Lemma \ref{L1},
Problem \eqref{single-agent} has the equivalent Dinkelbach's reformulation:
\begin{align}
{\pii}^*=\arg\mathop{\text{minimize}}_{{\pii}}\lim_{K\rightarrow \infty} \mathbb{E}_{\pii}\left[\sum_{k=0}^K\delta^{k} ( c_N- \gamma^* c_D)\right], \label{Din}
\end{align}
where we can see from Lemma \ref{L1} that  $\gamma^*$ satisfies
\begin{align}
    \gamma^*=\mathop{\text{minimize}}_{{\pii}}\lim_{K\rightarrow \infty}\frac{\mathbb{E}_{\pii}\left[\sum_{k=0}^K\delta^{k} c_N\right]}{\mathbb{E}_{\pii}\left[\sum_{k=0}^K\delta^{k} c_D\right]}.\label{eq-FRL-2}
\end{align}

Note that Problem \eqref{Din} is a classical MDP problem, \revr{including} an immediate cost  $c_N(\bs,\ba)-\gamma^*c_D(\bs,\ba)$. Thus, we can apply a traditional RL algorithm to solve such a reformulated problem, such as Q-Learning or its variants (e.g., SQL in \cite{SQL}).

However, the optimal quotient coefficient $\gamma^*$ and the transition distribution $P$ are not known \textit{a priori}. Therefore, we design an algorithm that combines fractional programming  and reinforcement learning to 
solve Problem \eqref{Din} for a given $\gamma$ and seek the value of $\gamma^*$. To achieve this, we start by introducing the following definitions.

    Given a quotient coefficient $\gamma$, the optimal Q-function is
    \begin{align}
        Q^*_\gamma(\bs,\ba)\triangleq \mathop{\rm{minimize}}_{\pii}\: Q^{\pii}_\gamma(\bs,\ba),~\forall (\bs,\ba)\in\mathcal{Z}, \label{optimalQ}
    \end{align}
where $Q^{\pii}_\gamma(\bs,\ba)$ is the action-state function that satisfies the following Bellman's equation: for all $(\bs,\ba)\in\mathcal{Z}$,
\begin{align}
    Q^{\pii}_\gamma(\bs,\ba)\triangleq c_N(\bs,\ba)-\gamma c_D(\bs,\ba)+\delta\mathbb{E}_{Pr}\left[Q^{\pii}_\gamma(\bs',\ba')\right],
\end{align}
where $\mathbb{E}_{Pr}$ is a concise notation for $ \mathbb{E}_{\bs'\sim Pr(\cdot|\bs,\ba)}$. In addition,  we can further decompose the optimal Q-function in \eqref{optimalQ} into the following two parts: $ Q^*_\gamma(\bs,\ba)=N_\gamma(\bs,\ba)-\gamma D_\gamma(\bs,\ba)$
and, for all $(\bs,\ba)\in\mathcal{Z}$, 
\begin{align}
    N_\gamma(\bs,\ba)&=c_N(\bs,\ba)+ \delta\mathbb{E}_{Pr}\left[N_\gamma(\bs',\ba')\right],\\ D_\gamma(\bs,\ba)&=c_D(\bs,\ba)+ \delta\mathbb{E}_{Pr}\left[ D_\gamma(\bs',\ba')\right].
\end{align}





    
\begin{algorithm}[t]
	\caption{Fractional Q-Learning (FQL)}\label{Algo-FQL}
	\begin{algorithmic}[1]
		\FOR{episode $i=0, 1,\cdots, E$}
		\STATE Initialize $\bs_0$;
		\FOR{task $k=0, 1,\cdots, K$}
		\STATE Observe the next state $\bs_{k+1}$;
		\STATE Observe a set of costs $\{c_{N,k},c_{D,k}\}$;
		\STATE Store $(\bs_k,\ba_k,c_{N,k},c_{D,k},\bs_{k+1})$ in replay buffer;
		\ENDFOR

		\STATE $\gamma_{i+1}=\frac{N_{\gamma_i}(\bs,\ba_i)}{D_{\gamma_i}(\bs,\ba_i)},$ where $\ba_i=\arg\min_{\ba} Q_{\gamma_i}(\bs,\ba)$.
		\ENDFOR
	\end{algorithmic}
\end{algorithm}

\subsection{Fractional Q-Learning Algorithm}
In this subsection, we present a Fractional Q-Learning (FQL) algorithm (see Algorithm \ref{Algo-FQL}). \revb{The algorithm consists of an inner loop with $E$ episodes and an outer loop.} The key idea is to approximate the Q-function \( Q_{\bgamma}^* \) using \( Q_i \) in the inner loop, while iterating over the sequence \( \{\gamma_i\} \) in the outer loop. 

A notable innovation in Algorithm \ref{Algo-FQL} is the design of the stopping condition, which ensures that the uniform approximation errors of \( Q_i \) shrink progressively. This allows us to adapt the convergence proof from \cite{dinkelbach67ms} to our inner loop, while proving the linear convergence rate of \( \{\gamma_i\} \) in the outer loop. Importantly, this design does not increase the time complexity of the inner loop.

We describe the details of the inner loop and the outer loop procedures of Algorithm \ref{Algo-FQL} in the following:
\begin{itemize}
    \item \textit{Inner loop}: For each episode $i$, given a quotient coefficient $\gamma_i$, we perform a Q-learning algorithm, such as Speedy Q-learning in \cite{SQL}, to approximate the function $Q_{\bgamma}^*(\bs,\ba)$ by $Q_i(\bs,\ba)$. Let $\s_0$ denote the initial state of any  episode, and $\ba_i\triangleq \arg\min_{\ba} Q_{i}(\bs_0,\ba)$ for all $i\in[E]\triangleq\{1,...,E\}$. We consider a \textit{stopping condition} 
    \begin{align}
         \epsilon_i< -\alpha Q_i(\bs_0,\ba_i),~\forall i\in[E], \label{SC}
    \end{align}
 \rev{where $\alpha>0$ is the error scaling factor. This stopping condition ensures that the algorithm terminates in} each episode $i$ with a bounded \textit{uniform approximation error}:
      $\norm{Q^*_{\bgamma}-Q_i}\leq \epsilon_i,~\forall i\in[E].$
    Operator $\norm{\cdot}$ is the supremum norm, which satisfies $ \lVert g \rVert\triangleq \max_{(\bs,\ba)\in\mathcal{Z}}|g(\bs,\ba)|$. 
   \rev{ Specifically,     we obtain $ Q_i(\bs_0,\ba_i)$, $N_i(\bs_0,\ba_i)$, and $D_i(\bs_0,\ba_i)$, which satisfy:
    \begin{align}
      Q_i(\bs_0,\ba_i)=N_i(\bs_0,\ba_i)-\gamma_iD_i(\bs_0,\ba_i).  \label{Eq-Qup}
    \end{align}}

    \item \textit{ Outer loop:} We update the quotient coefficient:
    \begin{align}
        \gamma_{i+1}=\mathop{\mathbb{E}}\limits_{\bs_0 \sim \mu_0}\left[\frac{N_i(\bs_0,\ba_i)}{D_i(\bs_0,\ba_i)}\right], \quad\forall i\in[E], \label{Update-O}
    \end{align}
   which will be shown to converge to the optimal value $\gamma^*$. 
\end{itemize}

\subsection{Convergence Analysis}

We present the time complexity analysis of inner loop and the convergence results of our proposed FQL algorithm (Algorithm \ref{Algo-FQL}) as follows.

\subsubsection{Time Complexity of the inner loop}
Although $\{Q_i(\bs_0,\ba_i)\}$ converges to $0$ and hence the stopping condition $\epsilon_i<-\alpha Q_i$ becomes more restrictive as $i$ increases, \revbb{the number of inner-loop steps $T_i$ in Algorithm~\ref{Algo-FQL} is finite and  non-increasing in $i$}. See \cite[Appendix A]{jin2024asynchronousfractionalmultiagentdeep} in detail.
\subsubsection{Convergence of FQL algorithm}
\revb{We then analyze the convergence of FQL algorithm:}
\begin{theorem}[Linear Convergence of Fractional Q-Learning]\label{T1}
\rev{If  the uniform approximation error} $\norm{Q^*_{\gamma_i}-Q_i}\leq \epsilon_i$ holds with $\epsilon_i< -\alpha Q_i(\bs_0,\ba_i)$ for  some $\alpha\in(0,1)$ and for all $i\in[E]$, then the sequence $\{\gamma_i\}$ generated by Algorithm \ref{Algo-FQL}  satisfies
\begin{equation}
    \frac{\gamma_{i+1}-\gamma^*}{\gamma_i-\gamma^*}\in(0,1),~\forall i\in[E] \text{ and }
    \lim_{i\rightarrow \infty} \frac{\gamma_{i+1}-\gamma^*}{\gamma_i-\gamma^*}=\alpha.
\end{equation}
That is, $\{\gamma_i\}$ converges to $\gamma^*$ linearly.
\end{theorem}
\revb{While the convergence proof in 
\cite{dinkelbach67ms} requires obtaining the \textit{exact} solution in each episode, Theorem \ref{T1} generalizes this result to the case where we only obtain an \textit{approximated (inexact)} solution in each episode. }
In addition to the proof techniques in \cite{dinkelbach67ms} and \cite{SQL},  our proof techniques include induction and
exploiting the convexity of $Q_i(\bs,\ba)$.
We present a proof sketch of Theorem \ref{T1} in \cite[Appendix B]{jin2024asynchronousfractionalmultiagentdeep}.


The significance of Theorem \ref{T1} is two-fold. First, Theorem \ref{T1} shows that Algorithm \ref{Algo-FQL} achieves a linear convergence rate, even though it only attains an approximation of \revr{$Q_\gamma^*(\bs,\ba)$\deleting{in each episode $i\in[E]$}}. Second,  \eqref{SC} is a well-behaved stopping condition.

\section{Fractional Multi-Agent RL Framework}\label{sec:ma_solution}

\rev{In this section, we propose a fractional MARL framework with inexact Newton method, extending our fractional framework from the single-agent scenario  to the multi-agent one. We first introduce a Markov game for our task scheduling problem. We then develop a fractional framework including fractional sub-games to address the fractional objective challenge. Finally, we propose the fractional Nash Q-learning algorithm and analyze the convergence of this algorithm to Nash equilibrium.  }

\subsection{Markov Game Formulation}\label{subsec:prob}
\edit{We study our problem in a multi-agent scenario and introduce the  task scheduling Markov game. 
The task scheduling Markov game is defined as: $G = (M, \mathcal{S}, \mathcal{A},\{c_{N,m}\}_{m=1}^{M}, \{c_{D,m}\}_{m=1}^{M},\mathcal{P}, \delta, \mu_0)$ where $M$ is the number of agents, $\mathcal{S}$ and $\mathcal{A} $ represent the state space and action space, respectively, $\mathcal{P}(\bs'|\bs,\ba)$ denotes the state transition probability, $\delta\in(0,1)$ is a discount factor, and $\mu_0$ denotes the initial state distribution.}




We define the instant fractional costs at task $k$ of mobile device $m$ as
\begin{align}
&c_{N,m}(\bs_k, \ba_k) =A(Y_{m,k},Z_{m,k+1},Y_{m,k+1}), \\
&c_{D,m}(\bs_k, \ba_k) = Y_{m,k}+Z_{m,k+1}.
\end{align}
\rev{Following the definition of discounted AoI in \eqref{discount-aoi}, we define the fractional long-term discounted cost  for }mobile device $m$  as
\begin{align}
    V_m(\bs, \pii) = \frac{\mathbb{E}_{\pii}\left[\sum\limits_{k=0}^{\infty} \delta^{k}c_{N,m}(\bs_k, \ba_k)  \middle| \bs_0=\bs\right]}{\mathbb{E}_{\pii}\left[ \sum\limits_{k=0}^{\infty} \delta^{k}c_{D,m}(\bs_k, \ba_k)\middle| \bs_0=\bs\right]},\label{cost_function}
\end{align}
\revb{where  $\pii = (\pii_m, \pii_{-m})$ represents the policies of mobile device $m$  and other agents that determine actions, and $\mathbb{E}_{\pii}$ is the expectation regarding transition dynamics given stationary control policy $\pii$.}  Note that \eqref{cost_function} approximates the un-discounted function when $\delta$ approaches 1 \cite{adelman_relaxations_2008}. 

\rev{Specifically, an NE is a tuple of policies $\pii^{\star} = (\pii_m^{\star}, \pii_{-m}^{\star})$ of game $\bG$. For each mobile  $m\in\M$, given the fixed stationary policies $\pii_{-m}^{\star}$,  the fractional objective can be expressed  as:}
\begin{align} \label{cost_problem}
\pii_m^{\star} =\mathop{\arg \mathop{\rm{minimize}}}\limits_{\pii_{m}} ~V_m(\bs, \pii_m, \pii_{-m}^{\star}), \forall \bs\in\mathcal{S}.
\end{align}

Lemma \ref{exist} guarantees the existence of an NE in  game $\bG$.
\begin{lemma}
     For the multi-player stochastic game with expected long-term discounted costs, there always exists a Nash equilibrium in stationary policies, where each agent follows a stationary policy whose action distribution does not change over time \cite{fink1964equilibrium}. 
    \label{exist}
\end{lemma}

\subsection{Fractional Sub-Game}
\rev{To deal with the fractional objective \eqref{cost_problem}  for NE, we reformulate our game as a sequence of fractional sub-games based on Dinkelbach's method\cite{dinkelbach67ms} in  this subsection.}

First, we define sub-game $G_{\bgamma}$ with fractional coefficients $\bgamma$ for fractional sub-games as follows.
\begin{definition}[Fractional Sub-Game]
The fractional sub-game is defined as $G_{\bgamma} = (M, \mathcal{S}, \mathcal{A}, \{c_{N,m}\}_{m=1}^{M}, \{c_{D,m}\}_{m=1}^{M}, \allowbreak \mathcal{P}, \delta, \bgamma) $, where the new components are
$\bgamma = \{\gamma_{m}\}_{m\in\mathcal{M}}$, representing the set of Dinkelbach variables. Each $\gamma_m$ is a Dinkelbach variable associated with mobile device $m$.
\end{definition}




\edit{The fractional value functions of sub-game are denoted as $\bV_{\bgamma}(\bs) :=[V_{\gamma_{m}}(\bs) ]_{m\in\M}$. The value function for each mobile device $m\in\mathcal{M}$ is given by
\begin{align}
   & V_{\gamma_{m}}(\bs, \pii) = (1-\delta)\cdot\nonumber\\
&\mathbb{E}_{\pii}\left\{\sum_{k=0}^{\infty}\delta^{k}[c_{N_{m}}(\bs_k, \ba_k) - \gamma_{m}c_{D_{m}}(\bs_k, \ba_k)]\middle|\bs_0=\bs\right\},
\end{align}}
where $\bs_k, \ba_k$ are the global state and action at task $k$.

In sub-game $G_{\bgamma}$, each mobile device $m$ aims to find its best response $\pii_{m}^{\star}$.  With  policies of other agents $\pii_{-m}^{\star}$ fixed, the best response of mobile device $m$ minimizes its  cost value $V_{\gamma_{m}}$ for any given global system state $\bs \in \mathcal{S}$. The objective is formulated as:
\begin{align} \label{sub_cost_problem}
\pii_{m}^{\star} =\mathop{\arg \mathop{\rm{minimize}}}\limits_{\pii_m} ~  V_{\gamma_{m}}(\bs, \pii_m, \pii_{-m}^{\star}),\quad \forall \bs \in \mathcal{S}.
\end{align}
Lemma \ref{exist} ensures the existence of an NE in sub-game $G_{\bgamma}$. For brevity, we denote the optimal state-value function as $V_{\gamma_{m}}^{\star}(\bs)=V_{\gamma_{m}}(\bs, \pii_{m}^{\star}, \pii_{-m}^{\star}), \forall \bs \in \mathcal{S}$, where $(\pii_{m}^{\star}, \pii_{-m}^{\star})$ is the NE of sub-game $G_{\bgamma}$. 

Thus, we propose a fractional Nash Q-Learning algorithm based on the iterative fractional sub-games and show its  convergence in the rest of this section.

\subsection{Fractional Nash Q-Learning Algorithm}

\rev{We introduce the Fractional Nash Q-Learning (FNQL) algorithm in Algorithm \ref{FNQL}. This algorithm includes two loops: \textit{1)} an inner loop that obtains NE of sub-game $G_{\bgamma}$, and \textit{2)}  an outer loop that iterates $\bgamma$ until convergence. The key idea is to iterate sub-game $G_{\bgamma}$  and  attain the NE of Markov game $\bG$ when $\bgamma$ converges.}
\subsubsection{Inner Loop} \label{inner_loop}
\rev{Following \cite{casgrain2022deep}, we first define the Nash operator $\mathop{\mathbb{N}}$ as follows:} 
\edit{\begin{definition}[Nash Operator]
    Consider a collection of $M$ functions, $\boldsymbol{f}(\ba) = [f_m(\ba_m,\ba_{-m})]_{m\in\M}$, which admit a Nash equilibrium $\pii^{\star}$. We introduce the Nash operator $\mathop{\mathbb{N}}\limits_{\ba\in\mathcal{A}}$, which maps the collections of functions into their corresponding Nash equilibrium values, i.e., $\boldsymbol{f}(\ba^{\star})= \mathop{\mathbb{N}}\limits_{\ba\in\mathcal{A}} \boldsymbol{f}(\ba)$,  satisfying
     \begin{align}
        f_m(\ba_m^*,\ba_{-m}^*) \leq f_m(\ba_m,\ba_{-m}^*), \quad \forall \ba_m, \forall m\in\M.
     \end{align}
\end{definition}}

For sub-game $G_{\bgamma}$, the optimization of best response \eqref{sub_cost_problem} \rev{utilizes an instant cost}, given by $c_{N_{m}}(\bs, \ba) - \gamma_{m}c_{D_{m}}(\bs, \ba),\forall m\in\M$. \rev{To solve this optimization, we apply the traditional MARL algorithm Nash Q-learning\cite{hu2003nash}.} 

\edit{
\textit{Nash-Q Functions.} We define the Nash Q-functions with $\bgamma$ as:

\begin{align}
\bQ_{\bgamma}(\bs,\ba)
\triangleq{}&
\bc_{N}(\bs,\ba)
-\bgamma\odot\bc_{D}(\bs,\ba)
\nonumber\\
&+\delta\,
\mathbb{E}_{\bs'}
\left[
    \mathop{\mathbb{N}}\limits_{\ba'\in\mathcal{A}}
    \bQ_{\bgamma}(\bs',\ba')
\right].
\label{NashQ}
\end{align}
where $ \odot$ stands for the element-wise product of two vectors. We denote  $\bc_{N}(\bs, \ba)=[c_{N_{m}}(\bs, \ba)]_{m\in\M}$,  $\bc_{D}(\bs, \ba)=[c_{D_{m}}(\bs, \ba)]_{m\in\M}$. We use $\mathbb{E}_{\mathcal{P}}$ as a shorthand notation for $ \mathbb{E}_{\bs' \sim \mathcal{P}(\cdot | \bs, \ba)}$. }
In the inner loop, we update the Q-function of mobile device $m\in\M$ at iteration $k$ by
\begin{align}\label{NQ-function}
\bQ_{\bgamma}^{k+1}(\bs, \ba)\! = & (1\!-\!\lambda)\bQ_{\bgamma}^{k}(\bs, \ba) + \lambda\left\{\bc_{N}(\bs, \ba) - \bgamma \odot \bc_{D}(\bs, \ba)\nonumber\right.\\
&\left. + \delta \mathbb{E}_{\bs'}\left[\mathop{\mathbb{N}}\limits_{\ba'\in\mathcal{A}}\bQ^k_{\bgamma}(\bs', \ba')\right]\right\},
\end{align}
where $\lambda$ is the learning rate. \edit{We observe that the Nash Q-function in \eqref{NashQ} is the fixed point of the iteration in \eqref{NQ-function}.}


\edit{\textit{Nash Equilibrium Policies.} Let $\pii^\star_{\boldsymbol{\gamma}}(\bs)=[\pi_m^\star(\bs)]_{m\in\mathcal{M}}$ be the NE policies of the sub-game $G_{\boldsymbol{\gamma}}$. We have that
\begin{align}
    \pi_{\boldsymbol{\gamma},m}^\star(\bs)=\arg\min_{\ba_m}Q_{\gamma,m}(\bs, \ba_m, \ba_{-m}^\star),\quad\forall m\in\mathcal{M}.
\end{align}}




\edit{\textit{Value functions.} We define  $\overline{\bN}^{\star}_{\bgamma}(\bs)=\left[\overline{N}_{\gamma_m}^{\star}(\bs)\right]_{m\in\M}$ and $\overline{\bD}^{\star}_{\bgamma}(\bs)=\left[\overline{D}^{\star}_{\gamma_m}(\bs)\right]_{m\in\M}$ as the numerator and denominator  value functions
as follows
\begin{align}\label{Nash_ND}  
     \overline{\bN}^{\star}_{\bgamma}(\bs)\triangleq
   \bc_{\bN}(\bs, \pii^*_{\boldsymbol{\gamma}}(\bs)) +\delta\mathbb{E}_{\bs' \sim \mathcal{P}(\cdot | \bs, \pii^*_{\boldsymbol{\gamma}}(\bs))}\left[\overline{\bN}^{\star}_{\bgamma}(\bs')\right], \\
\overline{\bD}^{\star}_{\bgamma}(\bs)\triangleq \bc_{\bD}(\bs, \pii^*_{\boldsymbol{\gamma}}(\bs)) +\delta\mathbb{E}_{\bs' \sim \mathcal{P}(\cdot | \bs, \pii^*_{\boldsymbol{\gamma}}(\bs))}\left[\overline{\bD}^{\star}_{\bgamma}(\bs')\right].
\end{align}
It follows from \eqref{NashQ} that $\mathop{\mathbb{N}}\limits_{\ba\in\mathcal{A}}\bQ_{\bgamma}(\bs, \ba)=\overline{\bN}^{\star}_{\bgamma}(\bs) - \bgamma \odot \overline{\bD}^{\star}_{\bgamma}(\bs)$.} We approximate the Nash operator $\mathop{\mathbb{N}}\limits_{\ba\in\mathcal{A}}$ as detailed in \cite[Appendix C]{jin2024asynchronousfractionalmultiagentdeep}.

\begin{figure}[t]
		\begin{algorithm}[H]
			\caption{Fractional Nash Q-Learning  (FNQL)}\label{FNQL}
			\begin{algorithmic}[1]
                \STATE Initialize a tolerance coefficient $\epsilon$;\
                \STATE Initialize $\bs_0$;\
				\REPEAT 				
				\FOR{$k=0, 1,\cdots, K$}
                \STATE Choose action $\ba$ based on current $\bQ^k$;\
                \STATE Observe $\bc_N,\bc_D$, and $\bs'$;\
                \STATE Update $\bQ^k_{\bgamma}$ by \eqref{NQ-function}; \
%
				\ENDFOR
				\STATE  Update $\gamma_{m}$ by \eqref{Update-Nash} for all $ m\in\M$;
				\UNTIL  $\norm{\mathbb{E}_{\boldsymbol{s}_0\sim \mu_0}\left[\overline{\bN}^{\star}_{\bgamma}(\bs_0) - \bgamma\odot\overline{\bD}^{\star}_{\bgamma}(\bs_0)\right]} \leq  \epsilon$.
			\end{algorithmic}
		\end{algorithm}
\end{figure}


 For each sub-game $G_{\bgamma}$, \rev{we perform the fractional Nash Q-Learning algorithm  to obtain} the Nash Q-functions  $\bQ_{\bgamma}(\bs)$. \rev{Each episode terminates when the Nash equilibrium of sub-game $G_{\boldsymbol{\gamma}}$ is attained.}

   \subsubsection{Outer loop} 

Given $\bgamma \in \Gamma$, we utilize neural networks (detailed in Section~\ref{sec:algorithm})  to approximate the Nash equilibrium strategy $\boldsymbol{\pi}^{\star}_{\bgamma,m}$ for each player $m \in \mathcal{M}$ in sub-game $G_{\bgamma}$. Let $\boldsymbol{\theta}_m(\bgamma) \in \Theta_m$ denote the parameters of the neural network trained for player $m$ given $\bgamma$. We define the mapping $\boldsymbol{\theta}: \Gamma \to \Theta$, where $\Theta = \times_{m \in \mathcal{M}} \Theta_m$. The outputs of these neural networks, parameterized by $\boldsymbol{\theta}(\bgamma)$, collectively approximate the NE strategy profile $\boldsymbol{\pi}^{\star}_{\bgamma}$ for the sub-game $G_{\bgamma}$.
Based on the strategies parameterized by $\boldsymbol{\theta}$, we define the following expected value functions for agent $m$:
\begin{align}
    N_m(\boldsymbol{\theta},\gamma_m)&= \mathbb{E}_{\boldsymbol{s}_0\sim \mu_0}\left[\overline{N}^{\star}_{\gamma_m}(\bs_0)\right],\\
    D_m(\boldsymbol{\theta},\gamma_m)&= \mathbb{E}_{\boldsymbol{s}_0\sim \mu_0}\left[\overline{D}^{\star}_{\gamma_m}(\bs_0)\right],\\
F_m(\boldsymbol{\theta},\gamma_m)&=\mathbb{E}_{\boldsymbol{s}_0\sim \mu_0}\left[\overline{N}^{\star}_{\gamma_m}(\bs_0) - \gamma_m \overline{D}^{\star}_{\gamma_m}(\bs_0)\right].
\end{align}


We denote the objective functions $\bF(\boldsymbol{\theta},\bgamma)=[F_m(\boldsymbol{\theta},\gamma_m)]_{m\in\M}$. For simplicity, let $\boldsymbol{\theta^*} = \boldsymbol{\theta}(\bgamma^*)$, which represents the parameters for an NE strategy of the original Markov game $\bG$. We establish the existence of an equivalent condition of reaching NE between the Markov game $\bG$ and the fractional sub-game $G_{\bgamma}$ with the following lemma:
\begin{lemma} \label{G-NE}
    There exists $\bgamma^* \succ \boldsymbol{0}$, such that  the NE strategy $\pii^*$ of  game $\bG$ is also an NE strategy of  sub-game $G_{\bgamma^*}$. For  $\bgamma^*$, the corresponding  $\bF(\boldsymbol{\theta}^*,\bgamma^*)$ satisfies:
    \begin{align}
    \bF(\boldsymbol{\theta}^*,\bgamma^*)=\boldsymbol{0}.
    \end{align} 
     \label{l_eqv}
\end{lemma}
Lemma \ref{l_eqv} is proven using the property of \eqref{NashQ} and
the Brouwer fixed point theorem. The detailed proof can be found in \cite[Appendix D]{jin2024asynchronousfractionalmultiagentdeep}. 

The outer loop implements the iterations of fractional sub-game  based on generalizing the Dinkelbach's method \cite{dinkelbach67ms}. At each iteration $i$,  after achieving the NE of sub-game $G_{\bgamma_i}$, we update  coefficients $\bgamma$ in line 9 of Algorithm \ref{FNQL} as
    \begin{align} 
\gamma_{m,i+1}=\frac{  N_m(\boldsymbol{\theta},\gamma_{m,i})}{  D_m(\boldsymbol{\theta},\gamma_{m,i})},~~\forall m\in\M.
\label{Update-Nash}
    \end{align}
We repeat this process until $\bgamma$ converges, which happens when $\mathbb{E}_{\boldsymbol{s}_0\sim \mu_0}\left[\overline{\bN}^{\star}_{\bgamma}(\bs_0) - \bgamma \overline{\bD}^{\star}_{\bgamma}(\bs_0)\right]$ is sufficiently close to $\boldsymbol{0}$. \revb{We proceed to show that the updates in \eqref{Update-Nash} constitute an inexact Newton's method and quantify the conditions that lead to a linear convergence rate for the outer loop.}

\subsection{Convergence Analysis of Outer Loop}
\revbb{We analyze the outer-loop convergence by connecting the update rule
\eqref{Update-Nash} to Newton's method. The outer loop seeks the root of
the composite mapping
$\bgamma \mapsto \bF(\boldsymbol{\theta}(\bgamma), \bgamma)=\boldsymbol{0}$,
where $\boldsymbol{\theta}(\bgamma)$ denotes the parameters of the NE strategy
for the sub-game $G_{\bgamma}$. We therefore consider the Jacobian of this
composite mapping with respect to $\bgamma$. Specifically, its $(m,n)$-th entry
is given by:}
\begin{equation} \begin{aligned}
   & [J_{\bF}(\boldsymbol{\theta},\bgamma)]_{mn}\\
    =&\sum_{k \neq m} \nabla_{\boldsymbol{\theta}_k} F_m(\boldsymbol{\theta}_m,\gamma_m) \frac{\partial \boldsymbol{\theta}_k(\boldsymbol{\theta})}{\partial \gamma_n} - \delta_{mn} D_m(\boldsymbol{\theta}_m,\gamma_m),
    \label{eq:true_jacobian}
    \end{aligned}
\end{equation} 
where $\delta_{mn}$ is defined as
$$
\delta_{mn} =
\begin{cases}
1, & \text{if } m = n, \\
0, & \text{if } m \neq n.
\end{cases}
$$
We start by analyzing a simplified scenario: the single-agent case without inter-agent dependencies. In this case, the Jacobian simplifies to a diagonal form:
\begin{equation} \begin{aligned}
\left[J_{\bF}'(\boldsymbol{\theta},\bgamma)\right]_{mm}= - D_m(\boldsymbol{\theta}_m,\gamma_m).
    \label{eq:approx_jacobian}
    \end{aligned}
\end{equation} 
Provided the required derivatives  are non-zero, the  update rule \eqref{Update-Nash} for agent $m$ in the exact Newton step for a single-agent case is
 \begin{align}
      \gamma_{m,i+1} =~&\gamma_{m,i} -\left[J_{\bF}'(\boldsymbol{\theta},\bgamma_i)\right]^{-1}_{mm}F_m(\boldsymbol{\theta}_m,\gamma_{m,i})\nonumber\\
      =~&\frac{  N_m(\boldsymbol{\theta},\gamma_{m,i})}{  D_m(\boldsymbol{\theta},\gamma_{m,i})}
 \end{align}
That is, the single-objective (single-agent) Dinkelbach's method is equivalent to the Newton's method, and hence
the sequence $\{\gamma_{m,i}\}$ converges quadratically to $\gamma_m^*$. 

\revb{However, generalizing to  the multi-agent scenario introduces significant theoretical and computational challenges, as it is difficult to compute $\frac{\partial \boldsymbol{\theta}_k(\boldsymbol{\theta})}{\partial \gamma_n}$ in the Jacobian \eqref{eq:true_jacobian}. }
For the multi-agent scenario, we consider  the update rule \eqref{Update-Nash}   within the framework of the inexact Newton method. The inexact Newton method is characterized by the following convergence conditions. 
\begin{lemma}
    \label{convergence_IN}
    Let $\boldsymbol{L}: \mathbb{R}^n \to \mathbb{R}^n$ be a continuously differentiable function with a solution $\boldsymbol{x}^*$ where $\boldsymbol{L}'(\boldsymbol{x}^*)$ is nonsingular. The inexact Newton iteration $\boldsymbol{x}_{i+1} = \boldsymbol{x}_i + \boldsymbol{s}_i$ and $\boldsymbol{L}'(\boldsymbol{x}_i) \boldsymbol{s}_i = -\boldsymbol{L}(\boldsymbol{x}_i) + \boldsymbol{r}_i$, converges linearly to $\boldsymbol{x}^\star$ if $\|\boldsymbol{r}_i\| \leq \eta_i\|\boldsymbol{L}(\boldsymbol{x}_i)\|$ for $0 \leq \eta_i \leq \eta_{\text{max}} < 1$, provided $\boldsymbol{x}_0$ is sufficiently close to $\boldsymbol{x}^*$ \cite{dembo1982inexact}.
\end{lemma}

In our context, the update step $s_i=\bgamma_{i+1}-\bgamma_i$  for the game $\bG$ can be viewed as the \textit{inexact} Newton equation:
\begin{equation} \label{eq:inexact_newton_eq}
    J_{\bF}(\boldsymbol{\theta},\bgamma_i) \bs_i = -F(\boldsymbol{\theta},\bgamma_i) + \boldsymbol{r}_i,
\end{equation}
where $\boldsymbol{r}_i$ is the residual term. This  allows us to adapt the convergence analysis of the inexact Newton method \cite{dembo1982inexact} to our multi-agent framework.

To establish convergence when employing an \textit{inexact} Newton method, we require several assumptions regarding the structure of the problem near the equilibrium. \revbb{ These assumptions are introduced to facilitate the inexact-Newton analysis of the outer-loop update near the equilibrium. } We define an open neighborhood $\mathcal{N}_{\bgamma^*}$ around $\bgamma^*$ and the subsequent assumptions hold uniformly for all $\bgamma \in \mathcal{N}_{\bgamma^*}$.


The first assumption guarantees desirable curvature properties for each agent's local objective component, related to the block-diagonal part of the Hessian of $\bF$ with respect to $\boldsymbol{\theta}$. We denote this Hessian as $H:= \nabla^2_{\boldsymbol{\theta}} \bF(\boldsymbol{\theta},\bgamma)$.

\begin{assumption}[Bounded Inverse of Local Hessian Block]\label{as:local_strong_convexity}
    For each $m\in\M$, we have $\|(H_{\boldsymbol{\theta}_m\boldsymbol{\theta}_m})^{-1}\|_\infty  \le 1/\mu$.
\end{assumption}
This assumption is a standard condition in optimization that prevents the Hessian blocks $H_{\boldsymbol{\theta}_m\boldsymbol{\theta}_m}$, which are related to each agent's local optimization problem, from being singular or near-singular.

We then impose conditions on the denominator term appearing in the average AoI metric, ensuring it is well-behaved and its gradient is bounded.
\begin{assumption}[Bounded Denominator and Gradient]
    \label{as:bounded_denominator}
    For all $m\in\M$, $D_m(\boldsymbol{\theta},\gamma_m)$ is uniformly bounded below by $D_{\text{min}} > 0$, i.e., $D_m(\boldsymbol{\theta},\gamma_m) \ge D_{\text{min}} > 0$. Additionally, the norm of the gradient of $D_m$ with respect to $\boldsymbol{\theta}_m$ is bounded: $\|\nabla_{\boldsymbol{\theta}_m} D_m(\boldsymbol{\theta},\gamma_m)\|_\infty  \le C_D$.
\end{assumption}
This assumption prevents division by zero or numerically unstable small values in the objective function $F_m$, ensuring it remains well-defined and finite. The bound $C_D$ on the gradient limits the sensitivity of the denominator to changes in $\boldsymbol{\theta}_m$.

We also need to control the strength of interactions between different agents.
\begin{assumption}[Bounded Interaction]
    \label{as:bounded_interaction} 
    For all $k\neq m$, the norm of the first- and second-order cross-gradients of $F_m$ is  bounded as $\|\nabla_{\boldsymbol{\theta}_k} F_m(\boldsymbol{\theta})\|_\infty  \le C_{\text{int}}$ and   $\|H_{\boldsymbol{\theta}_k \boldsymbol{\theta}_m} \| = \|\nabla^2_{\boldsymbol{\theta}_k \boldsymbol{\theta}_m} F_m(\boldsymbol{\theta}, \gamma_m)\|_\infty  \le H_{\text{int}}$ for all $m, k$.
     For each agent $m$, the number of other agents $k\neq m$ that directly influence agent $m$ through the gradient term is bounded as $|\mathcal{N}_m^{\text{grad}}| := |\{ k \neq m \mid \nabla_{\boldsymbol{\theta}_k} F_m(\boldsymbol{\theta},\bgamma) \neq 0 \}| \le K_{\text{grad}}$. \revbb{Similarly, the number of agents $k\neq m$ that influence agent $m$ through the cross-Hessian term is bounded by} $|\{ k \neq m \mid H_{\boldsymbol{\theta}_k \boldsymbol{\theta}_m}  \neq 0 \}|\leq K_{\text{hess}}$. We define $K = \max(K_{\text{grad}}, K_{\text{hess}})$ and assume $K$ is a constant independent of the system size $M$.
\end{assumption}
\revbb{This assumption bounds cross-agent coupling in the local convergence analysis by requiring the direct marginal effect of agent $k$'s parameters $\boldsymbol{\theta}_k$ on agent $m$'s objective $F_m$ to be limited. It implies that each agent is directly affected by at most $K$ other agents, independent of the total number of agents $M$. This is reasonable in large-scale MEC systems \cite{hu2015mobile}, where interactions are typically induced by shared queues and communication resources and are often more localized than fully dense. Nevertheless, the assumption may be restrictive in strongly coupled or highly heterogeneous systems.}

Based on Lemma \ref{convergence_IN} and  Assumptions \ref{as:local_strong_convexity}-\ref{as:bounded_interaction}, we are ready to present the convergence result for Algorithm \ref{FNQL} applied to game $\bG$:

\begin{theorem}[Linear Convergence of the Outer Loop] \label{T2}
        \revbb{Assuming Assumptions~\ref{as:local_strong_convexity}-\ref{as:bounded_interaction} hold,  the initial $\bgamma_0$ is sufficiently close to $\bgamma^*$, and the conditions $\frac{KH_{\text{int}}}{\mu} <1$ and $\frac{ C_D C_{\text{int}}}{ D_{\text{min}} (1/K- H_{\text{int}})} < 1$ hold, then the outer-loop update in  Algorithm \ref{FNQL} , characterized by the inexact Newton step in  \eqref{eq:inexact_newton_eq} converges locally linearly to  $\bgamma^*$. The corresponding equilibrium strategy is a Nash equilibrium of game  $\bG$.}
\end{theorem}

Please refer to \cite[Appendix E]{jin2024asynchronousfractionalmultiagentdeep} for a detailed proof. 
\revbb{The outer-loop convergence analysis, which accounts for multi-agent interactions and the inexact Newton method, constitutes a key contribution of this work in Theorem \ref{T2}.} This analysis is achieved under reasonable and justifiable assumptions about the problem structure. \revbb{Theorem \ref{T2} shows convergence to a Nash equilibrium of the decentralized game, rather than global optimality under a centralized network-wide AoI criterion.}

\section{Asynchronous Fractional MADRL Algorithm}\label{sec:algorithm}
\revbb{In this section, we introduce A. F. MADRL, an asynchronous fractional MADRL algorithm for approximating a Nash-stable policy profile of the SMG in Problem~\eqref{problem} with hybrid action spaces, as shown in Fig.~\ref{fig:AMADRL}. It serves as a deep asynchronous instantiation of the FNQL framework in Algorithm~\ref{FNQL}, where the Nash Q-functions in~\eqref{NashQ} are approximated by GRU-conditioned hybrid networks under CTDE.}

Building on the fractional MARL framework proposed in Section \ref{sec:ma_solution}, we first introduce a fractional cost module for updating $\bgamma$. This module approximates an NE when optimizing the fractional AoI objective in a multi-agent scenario.
However, when considering the MEC scheduling problem formulated as an SMG, synchronized decision-making and training become inefficient due to asynchronous task completions.  To address the challenge of asynchronous decision-making, we introduce an asynchronous trajectory collection mechanism. Additionally, we propose the corresponding hybrid strategy and value networks to apply the mechanism to hybrid action spaces of updating and offloading actions.

\begin{figure*}[t]
	\centering
 	\includegraphics[width=18cm]{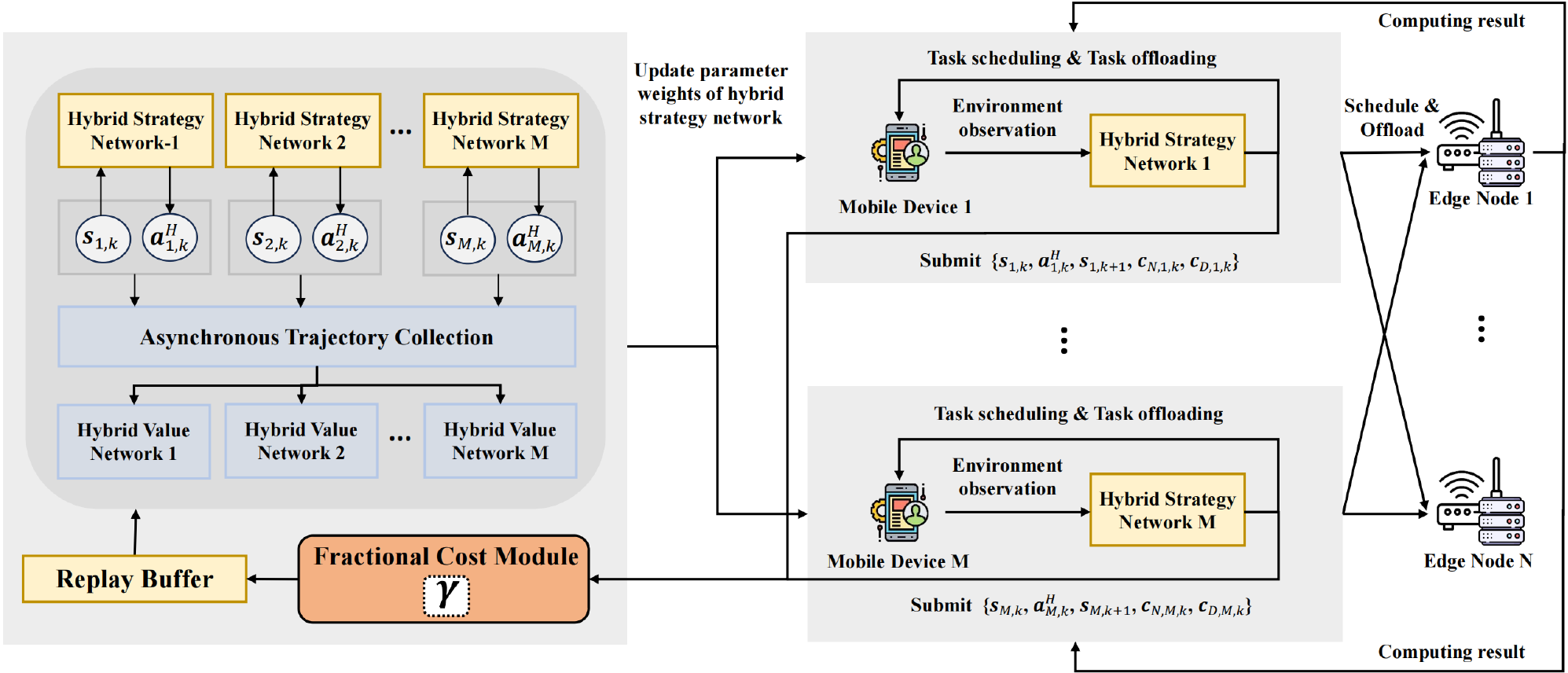}
	\caption{\revbb{Illustration of the proposed asynchronous fractional MADRL framework. The fractional cost module, asynchronous trajectory collection, and GRU-based history aggregation are used for centralized training, while the learned hybrid strategy networks are executed independently by individual mobile devices.}
	}
 \label{fig:AMADRL}
\end{figure*}

\subsection{Fractional Cost Module}
As in the proposed fractional MARL framework in Section \ref{sec:ma_solution}, we consider a set of episodes $i\in[E]$ and introduce a quotient coefficient $\gamma_i$ for episode $i$.  We define 
$\ba_{i}^{H}=[\ba_{m,i}^{H}]_{m=1}^{M}$, where $\ba_{m,i}^{H}\triangleq [(\ba_{m}^{\textsc{U}}, \ba_{m}^{\textsc{O}})]_{k=1}^{K}$ as the set of  updating and offloading strategies of mobile device $m\in\M$ in episode $i$, where superscript `H' refers to `history'. 


At each task  $k$, we determine a fractional cost and send it for training.  This process corresponds to the sub-game of the proposed fractional RL framework. Based on \eqref{sub_cost_problem}, we define the cost for all  $ i\in[E], m\in\mathcal{M}$, and $ k\in\mathcal{K}$ as:
\begin{align}
    c_{m,k}^{i}(\ba^{H}_{m},\ba^{H}_{-m})=& A(Y_{m,k}^i,Z_{m,k+1}^i,Y_{m,k+1}^i)\nonumber\\
    &- \gamma_{m}  \cdot (Y_{m,k}^i+Z_{m,k+1}^i), 
    \label{eq:cost-DRL}
\end{align}
where for episode $i$, $Y_{m,k}^i$ denotes the duration of task $k$, $Z_{m,k+1}^i$ is the waiting time before generating task $k+1$, and $A(Y_{m,k}^i,Z_{m,k+1}^i,Y_{m,k+1}^i)$ represents the trapezoid area in \eqref{trapezoid}. Note that $Y_{m,k}$ is a function of $\ba_{m}^{\textsc{O}}$, and $Z_{m,k+1} = a_{m,k+1}^{\textsc{U}}$. The fractional cost module tracks $(\ba^{H}_{m},\ba^{H}_{-m})$, or equivalently $Y_{m,k}$ and $Z_{m,k}$ for all $k=0,1,...,K$  throughout   the training process.



Finally, at the end of each episode $i$, the fractional cost module updates $\gamma_{m,i+1}$ by:
 \begin{align}
    \gamma_{m}^{i+1}=\frac{ \sum\limits_{k=0}^{K}\delta^{k}A_m(Y_{m,k}^{i},Z_{m,k+1}^{i},Y_{m,k+1}^{i})}{\sum\limits_{k=0}^{K}\delta^{k}(  \revr{Y_{m,k}^{i}}+Z_{m,k+1}^{i})}. \label{eq:gammaupdate}
\end{align}

\subsection{Asynchronous Trajectory Collection Mechanism}


\revbb{Our framework adopts CTDE to support asynchronous actions while preserving decentralized execution. Specifically, centralized training aligns asynchronous events via a global event index $T$ and constructs globally consistent samples for policy and value learning. At execution time, each mobile device acts independently using only local information.}

We compare our asynchronous mechanism with traditional MARL and VarlenMARL \cite{chen2021VarlenMARL}, an asynchronous MADRL algorithm, in the following example:
\begin{itemize}
    \item \edit{Traditional MADRL algorithms such as MADDPG \cite{lowe2017multi} and MAPPO require synchronous joint information from all agents at each time step, making them unsuitable for asynchronous decision-making. Traditional MARL algorithms gather  joint data from all agents at each time step.} The collected joint trajectory is
\begin{align}
    [(s_{1}, \ba_{1,1},\ba_{2,1}), (s_{2}, \ba_{1,2},\ba_{2,2}), (s_{3}, \ba_{1,3},\ba_{2,3}), \cdots].
\end{align}
\item VarlenMARL \cite{chen2021VarlenMARL}, an asynchronous MADRL algorithm,  addresses this by allowing variable-length trajectories. However, it still struggles to fully utilize agent interactions in asynchronous environments. It  allows agents to collect trajectories asynchronously. It merges  an agent's data with the latest padded data from other agents.  Let $\bs_m^{0}, \ba_m^{0}$ be the initial trajectory  of agent $m$. Fig.  \ref{fig:trajectory}(a) shows the trajectories of Agents $1$ and $2$ as:
\begin{align}
    &[(\s_{1,1}, \ba_{1,1},(\s_{2,0}, \ba_{2,0})),(\s_{1,2}, \ba_{1,2},(\s_{2,1}, \ba_{2,1})),\nonumber \\
    &(\s_{1,3}, \ba_{1,3},(\s_{2,1}, \ba_{2,1})) \cdots], [(\s_{2,1}, \ba_{2,1},(\s_{1,1},\ba_{1,1})),\nonumber \\
    & (\s_{2,2}, \ba_{2,2},(\s_{1,3}, \ba_{1,3})), (\s_{2,3}, \ba_{2,3},(\s_{1,3}, \ba_{1,3})), \cdots].\nonumber 
\end{align}

\item \revbb{As shown in Fig. \ref{fig:trajectory}(b), during centralized training, a global event sequence is maintained to align asynchronous experience tuples collected from different agents.  Let $(\bs_{T}, \ba_{T})$ denote the state-action pair associated with the $T$-th global event, generated by agent $m_T$.} To capture the history of interactions among agents, we introduce the time-aggregated history state, $H_T \in \mathbb{R}^d$. $H_T$ is a fixed-dimensional vector  that serves as a summary representation derived from the sequence of all state-action pairs corresponding to global events from $1$ to $T-1$. This history state is updated sequentially using a history aggregation transition function, denoted as $TA$. Upon receiving the $T$-th state-action pair $(\bs_{T}, \ba_{T})$, the controller computes the subsequent history state $H_{T+1}$ via the update rule:
\begin{align} \label{eq:history_update_revised}
    H_{T+1} = TA\left( (\bs_{T}, \ba_{T}), H_T \right)
\end{align}
where $TA: (\mathcal{S} \times \mathcal{A}) \times \mathbb{R}^d \to \mathbb{R}^d$ represents the transition function that maps the current event's state-action information and the previous history state to the new history state. In our implementation, $TA$ is implemented by a GRU network \cite{chung2014empirical}, processing an embedding of the input pair $(\bs_{T}, \ba_{T})$ and the previous hidden state $H_T$ to produce $H_{T+1}$. Implementation details of the RNN networks are provided in Appendix F \cite{jin2024asynchronousfractionalmultiagentdeep}. The resulting  collected trajectory is a sequence ordered by the global event index $T$, as shown in Fig.~\ref{fig:trajectory}(b):
\begin{align}
    &[(\s_{1,1},\ba_{1,1},H_1),(\s_{2,1},\ba_{2,1},H_2),(\s_{1,2},\ba_{1,2},H_3),\nonumber \\
    &(\s_{1,3},\ba_{1,3},H_4),(\s_{2,2},\ba_{2,2},H_5),(\s_{2,3},\ba_{2,3},H_6), \cdots].\nonumber
\end{align}
This mechanism facilitates more accurate evaluations of states and $Q$ values by leveraging global states, actions, and the aggregated information derived from all agents' prior actions.
\end{itemize}
\begin{figure}
    \centering
    \includegraphics[height=2.3cm]{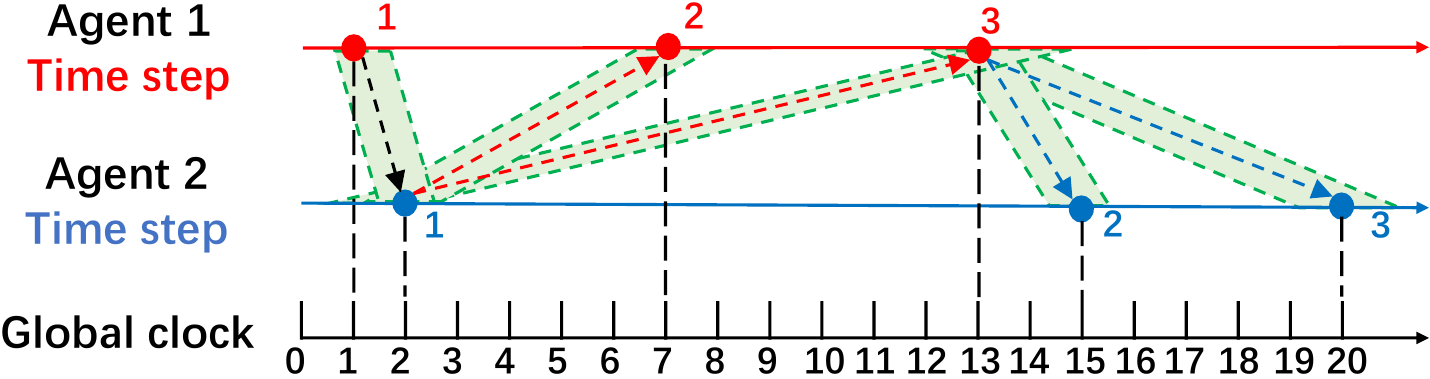} \\
    (a)
    \includegraphics[height=2.3cm]{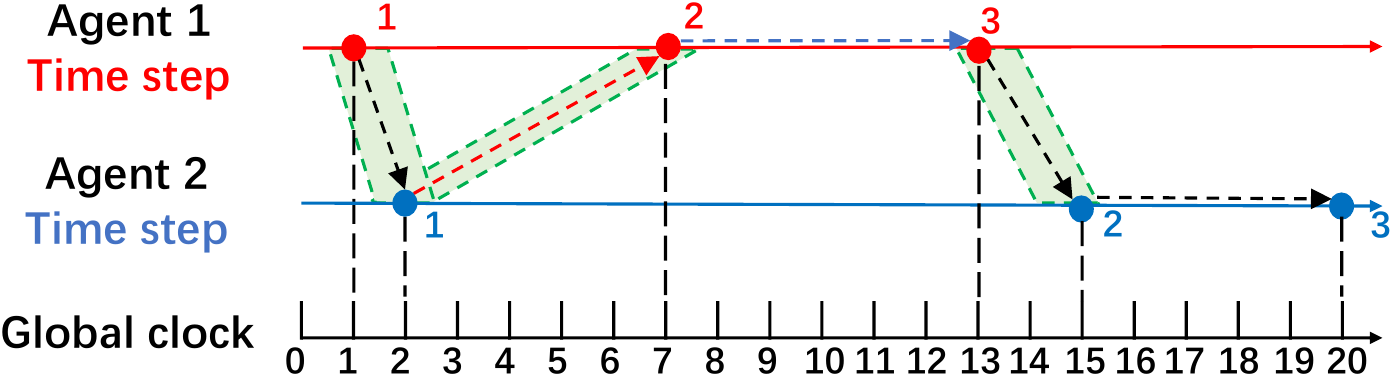} \\
    (b)
    \caption{Comparison of trajectory collection mechanisms under VarlenMARL, and our proposed RNN-PPO. (a) In VarlenMARL, an agent collects its trajectory as well as the latest information from other agents. 
    (b) In RNN-PPO, an agent collects its trajectory with the time-aggregated information from historical decision events.  
    }
    \label{fig:trajectory}
\end{figure}

\subsection{Hybrid Strategy Network}

Our algorithm employs independent learners to train hybrid strategy and value networks for each device asynchronously. \revbb{Within the CTDE framework, the GRU-based shared history is used only during centralized training; at execution time, each agent acts using only its locally available information.}\footnote{\revbb{Training is performed centrally on a trusted server; only the learned policy parameters are deployed to devices.}} 

The asynchronous trajectory collection mechanism synthesizes time-aggregating data. It utilizes state-action pairs from other agents and historical events to improve agent interactions. \revb{We implement this mechanism with a GRU, forming the RNN-D3QN and RNN-PPO architectures for discrete and continuous actions, respectively. Please refer to Appendix F for details.} \revb{On a central server, the collected asynchronous trajectories are sampled to perform synchronous parameter updates for all agents' networks.} \revb{The resulting off-policy learning challenge is stabilized by the inherent design of our DRL algorithms (e.g., importance sampling in PPO) and by conditioning updates on the aggregated historical context from the GRU.} These hybrid networks address both the task offloading and updating processes.



\section{Performance Evaluation}
\revbb{Unless otherwise specified, the default setting uses 20 mobile devices. We further vary the number of agents to evaluate scalability.} \revr{Our experimental setup  follows the settings outlined in \cite[Table I]{Tang2020TMC}, with key experiment settings and parameters provided in \cite[Appendix G]{jin2024asynchronousfractionalmultiagentdeep}.}
We compare our proposed  algorithms and baseline algorithms as follows.
\begin{itemize}
    \item \textit{Non-Fractional MADRL (denoted by Non-F. MADRL)}: \revb{This benchmark lets each mobile device learn hybrid DRL algorithms individually, without a fractional scheme.} \edit{The algorithm leverages D3QN  to generate discrete  offloading decisions, while employing PPO to produce continuous updating decisions. In comparison to our proposed fractional framework, this benchmark is non-fractional and lacks the asynchronous trajectory collection mechanism.} It approximates the  ratio-of-expectation average AoI by an expectation-of-ratio expression: $\mathop{\text{minimize}}_{\pii_{m}}~~ \mathbb{E}\left[\left.\frac{ A(Y_{m,k},Z_{m,k+1},Y_{m,k+1})}{ Y_{m,k}+Z_{m,k+1}}\right|\pii_m\right].$ This approximation to tackle the fractional challenge tends to cause a large loss in accuracy.
    \item \textit{Asynchronous Non-Fractional MADRL (denoted by A. Non-F. MADRL)}: This benchmark incorporates the asynchronous trajectory collection mechanism into the Non-F. MADRL algorithm to demonstrate its effectiveness. While it employs the same asynchronous trajectory collection mechanism as our proposed A. F. MADRL algorithm, it lacks the fractional framework.
    \item  \textit{A. F. MADRL}: The proposed A. F. MADRL algorithm incorporates two key elements. First, it uses a fractional scheme to approximate the NE. Second, it employs a trajectory collection mechanism, which aggregates information from both historical records and other agents.
    \item \textit{Random}: This method randomly schedules updating and offloading decisions within the action spaces. 
    \item \textit{H-MAPPO}: Multi-agent proximal policy optimization \cite{yu2022surprising} with hybrid action spaces. 
    \item \textit{VarlenMADRL}: Adapted VarlenMARL (denoted by VarlenMADRL): We adapt the variable-length trajectory collection mechanism of VarlenMARL \cite{chen2021VarlenMARL} to our deep multi-agent reinforcement learning setting with hybrid action spaces.
    \item \textit{F. DRL}: Fractional DRL (F. DRL) \cite{jin2024fractional} adapts a fractional framework based on Dinkelbach's method to DRL algorithms without considering the multi-agent scenario.
    \revb{\item  \textit{HAPPO}:  Heterogeneous-Agent Proximal Policy Optimisation (HAPPO) algorithms \cite{kuba2021trust} with hybrid action spaces.}
    \revb{\item \textit{PGOW}: Potential-game-inspired offloading \& waiting baseline (PGOW) that combines $\epsilon$-greedy best-response offloading with log-linear learning for waiting. The offloading strategy considers estimated service time and queue-dependent congestion. These dynamics are standard in potential game settings \cite{monderer1996potential,blume1993statistical}.}
    \item \textit{ASM}: Asynchronous and Scalable Multi-agent Proximal Policy Optimization (ASM-PPO) \cite{liang2022asm} allows asynchronous learning and decision-making based on PPO.

\end{itemize}


\begin{figure}[t]
\centering
\includegraphics[width=4.35cm]{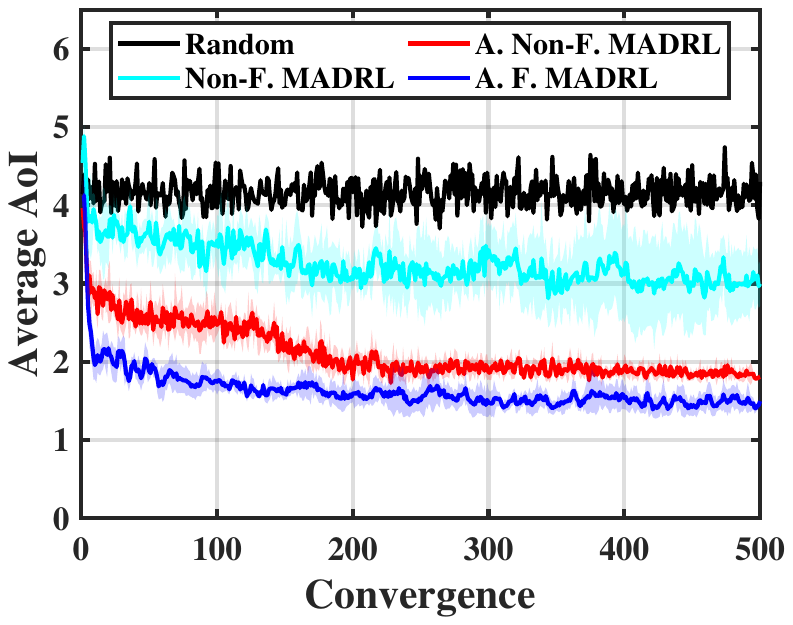}\label{fig:aoi}
\includegraphics[width=4.35cm]{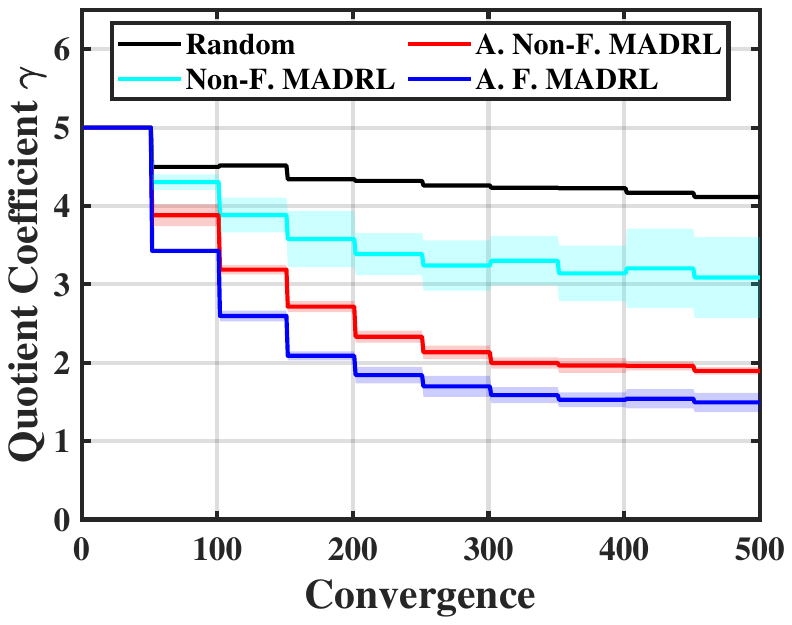}\label{fig:gamma}\\
\quad(a)\qquad\qquad\qquad\qquad\qquad\qquad(b)
\caption{
\edit{Convergence of (a)  average AoI and (b) the exponential moving average values of (Dinkelbach) quotient coefficients $\bgamma$ across devices, where $\bgamma$ is updated every $50$ episodes. The unit of AoI is measured in seconds.}}
\label{fig:convergence}
\end{figure}

Our experiments first evaluate the convergence of our proposed algorithms. We then compare their performance against  baselines. \edit{
The average AoI during evaluation is computed using the following equation:
\begin{align}
     \Delta_m=\frac{ \sum\limits_{k=0}^{K}A_m(Y_{m,k},Z_{m,k+1},Y_{m,k+1})}{\sum\limits_{k=0}^{K}( {Y_{m,k}}+Z_{m,k+1})}.
\end{align}
}
Our experiments systematically evaluate the impacts of several factors:  edge capacity, drop coefficient, task density, mobile capacity, processing variance, number of agents, and channel bandwidth.

\textbf{Convergence}: \edit{Fig. \ref{fig:convergence} illustrates the convergence result of the Non-F. MADRL algorithm, as well as the proposed A. Non-F. MADRL and A. F. MADRL algorithms. We examine two key aspects: the convergence of the average AoI and that of the average fractional coefficient, $\gamma$, which should converge concurrently.}

\edit{
In Fig. \ref{fig:convergence}(a), both A. Non-F. MADRL and A. F. MADRL algorithms are shown to converge after approximately 300 episodes. \revb{Notably, A. F. MADRL achieves a consistent reduction in converged average AoI compared to A. Non-F. MADRL, demonstrating the efficacy of the proposed fractional framework.} Furthermore, the significant performance gap between A. Non-F. MADRL and Non-F. MADRL shows the  impact of the asynchronous trajectory collection mechanism. \revbb{On the other hand, the random scheme does not learn; its quotient coefficient  in Fig.~\ref{fig:convergence}(b) gradually approaches the stationary mean in Fig.~\ref{fig:convergence}(a), starting from its initial value.}
}

\begin{figure*}[t]
    \centering
    \includegraphics[height=4.52cm]{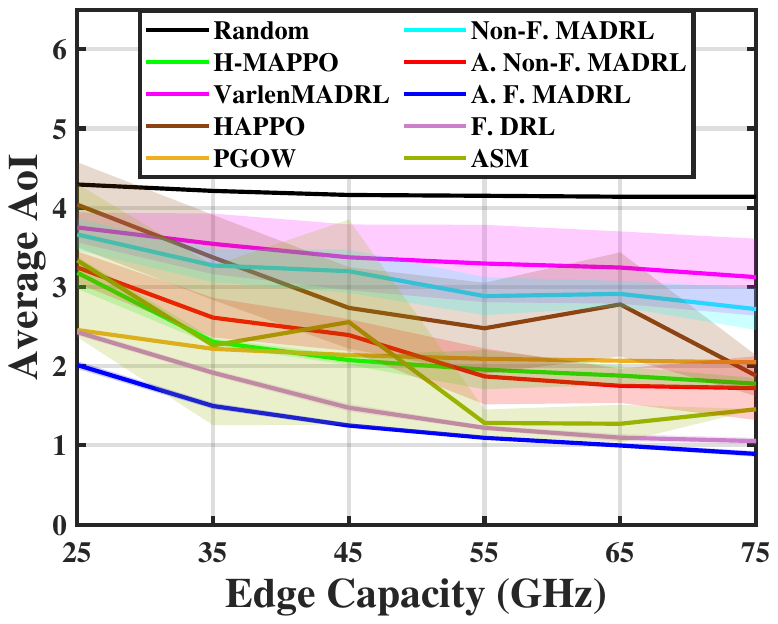}\label{fig:edge}
    \includegraphics[height=4.52cm]{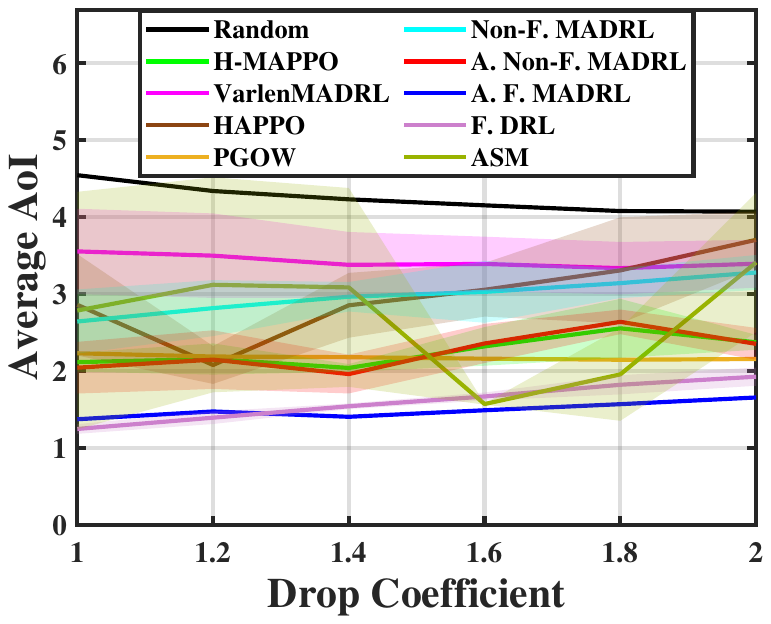}\label{fig:drop}
    \includegraphics[height=4.52cm]{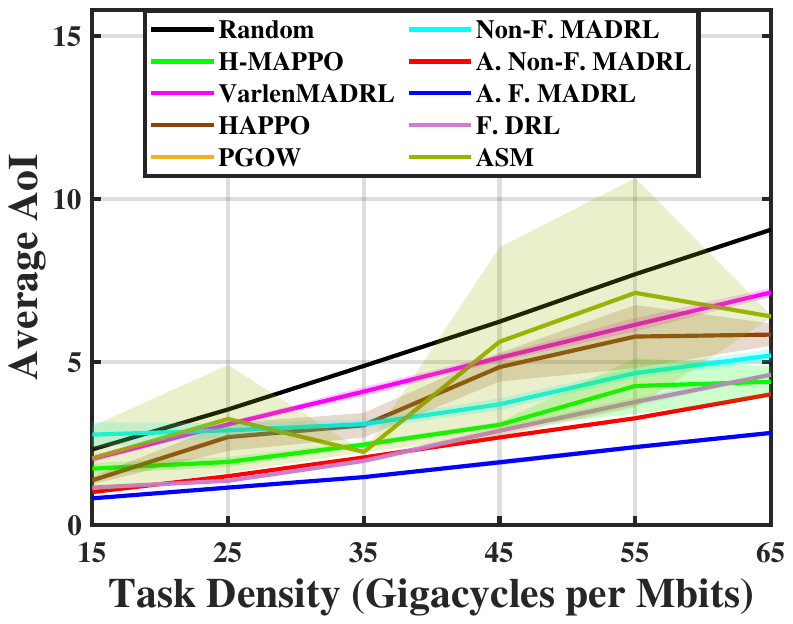}\label{fig:density}\\
    (a)\qquad\qquad\qquad\qquad\qquad\qquad\qquad\qquad(b)\qquad\qquad\qquad\qquad\qquad\qquad\qquad\qquad(c)\\
    \includegraphics[height=4.5cm]{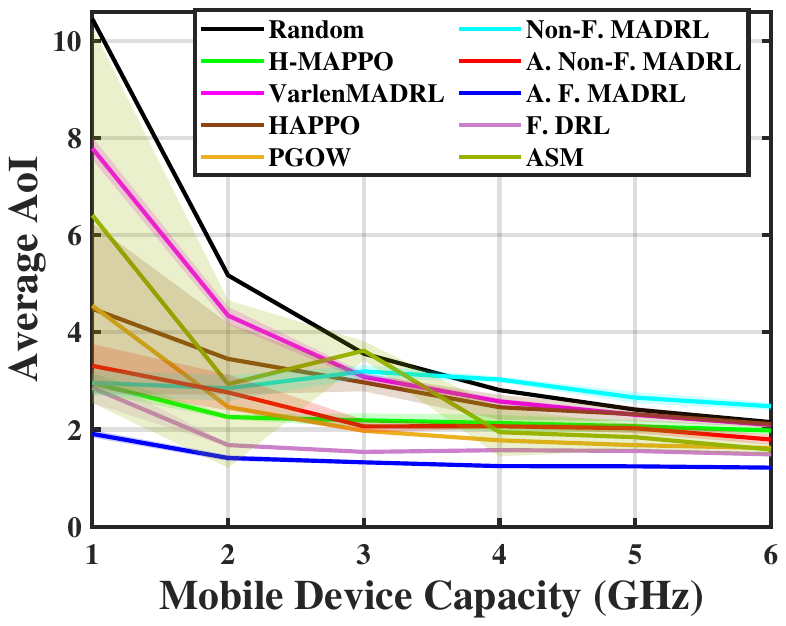}\label{fig:mobile}
    \includegraphics[height=4.45cm]{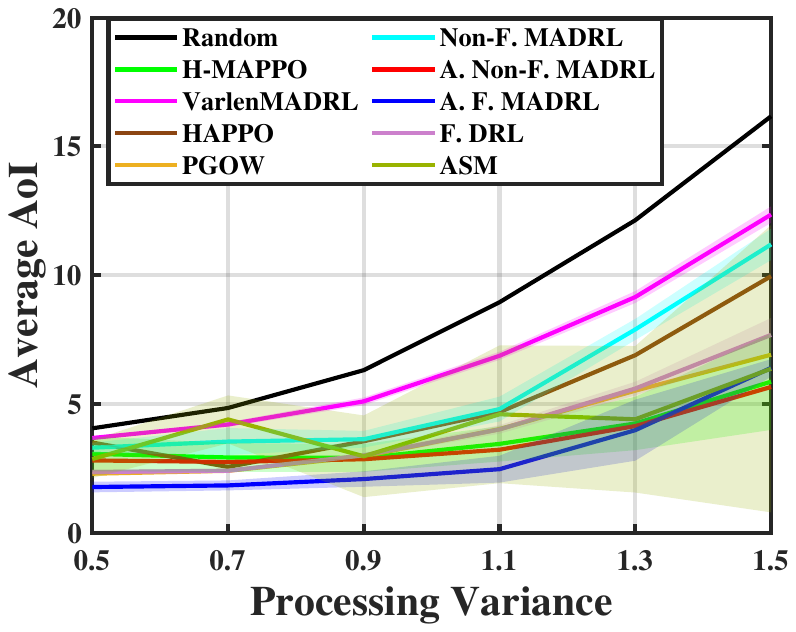}\label{fig:agent}
    \includegraphics[height=4.5cm]{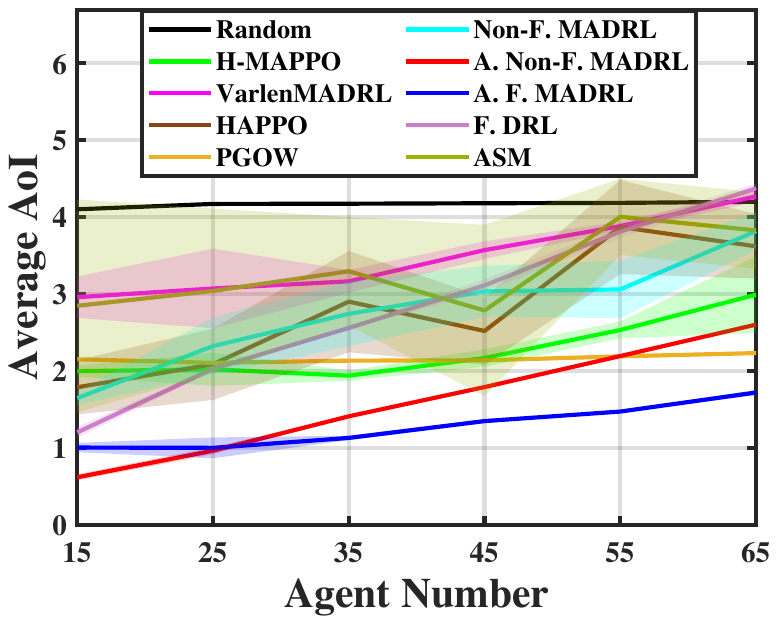}\label{fig:var}\\
    (d)\qquad\qquad\qquad\qquad\qquad\qquad\qquad\qquad(e)\qquad\qquad\qquad\qquad\qquad\qquad\qquad\qquad(f)
    
    \caption{
    \edit{Performance  under different  (a) processing capacities of edges,
    (b) drop coefficients, (c) task densities,  (d) processing capacities of mobile devices, \\(e) processing variances, and (f) numbers of agents. }
    }
    \label{fig:compare} 
\end{figure*}

\textbf{Edge Capacity}: In Fig. \ref{fig:compare}(a), we evaluate different processing capacities of edge servers. \revb{The use of the fractional cost module in A. F. MADRL results in an average AoI reduction of 31.1\% compared to the A. Non-F. MADRL approach.} The proposed A. F. MADRL algorithm consistently achieves smaller average AoI  compared to other benchmarks that do not consider fractional AoI and asynchronous settings. \revb{A. F. MADRL reduces the average AoI by up to 21.9\% compared with F. DRL at 35 GHz. We observe that the performance gap between the proposed A. F. MADRL and other baseline schemes becomes more pronounced when the edge capacity is limited. These results demonstrate that an optimized design tailored for asynchronous settings can achieve superior performance, especially in resource-constrained environments. }

\textbf{Drop Coefficient}: Fig. \ref{fig:compare}(b) illustrates  the results of different drop coefficients, which are the ratios of drop time $\bar{Y}$ to the task processing duration. While the average AoI of random scheduling decreases as the drop coefficient increases, other RL techniques show no clear trend. This suggests that the updating policy may serve as a dropping mechanism.   \revb{With different drop coefficients, our proposed A. F. MADRL algorithm achieves approximately 33.8\% lower average AoI than H-MAPPO, on average.}

\textbf{Task Density}: Fig. \ref{fig:compare}(c) shows the AoI performance under varying task densities.  These densities affect overall task loads and expected processing duration at mobile devices and edge servers.  \revbb{A. F. MADRL outperforms Non-F. MADRL by 26.2\% on average.} Additionally, our A. F. MADRL algorithm outperforms the best existing baselines by 42.0\% on average. At high task density (65 Gigacycles per Mbit), A. F. MADRL reduces the average AoI by up to 38.8\% compared to F. DRL.  These results demonstrate the effectiveness  of our asynchronous trajectory collection mechanism, especially when computational resources are highly restrained.

\textbf{Mobile Capacity}:  Fig. \ref{fig:compare}(d) presents results for  different mobile device capacities. \revb{Our A. F. MADRL algorithm reduces average AoI by 20.4\% on average compared to F. DRL algorithm.} \edit{The performance differences diminish as mobile device computational capacities increase, primarily because devices with higher processing capabilities can handle most tasks locally.} Consequently, the effectiveness of task scheduling diminishes. 

\textbf{Processing Variance}: Fig. \ref{fig:compare}(e) demonstrates the performance of  algorithms under different variances of processing time. In this particular experiment, processing times are modeled using a lognormal distribution.  \revb{Our A. F. MADRL algorithm outperforms F. DRL algorithm by up to 30.7\% when the processing variance is 1.1. These results demonstrate the robustness of our asynchronous trajectory collection mechanism under lognormal processing time distributions.  }

\textbf{Number of Mobile Devices}: Fig. \ref{fig:compare}(f) shows the performance with varying numbers of mobile devices and fixed edge servers. \revbb{Specifically, under the large-scale setting with 65 agents, where competition for edge-server queues is severe, A. F. MADRL improves performance by up to 34.0\% over A. Non-F MADRL and by 41.0\% on average over the best existing baselines.} These results demonstrate that our proposed framework scales robustly as competition for resources intensifies.

\textbf{Bandwidth of Time-Varying Channels:} We compare the average AoI  between proposed A. F. MADRL algorithm and other baselines in Section VII under different bandwidths. We set the channel parameters as in Table III, Appendix G.
\begin{figure}
    \centering
    \includegraphics[width=6cm]{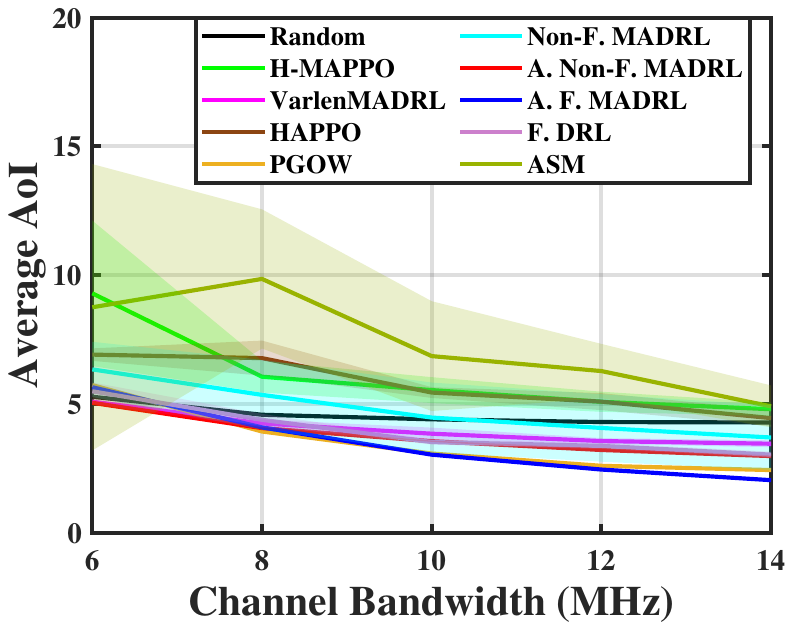} 
    \caption{\revbb{Performance under varying bandwidths with time-varying channels.}}\label{fig:channel}
\end{figure}
\revb{As shown in Fig. \ref{fig:channel}, traditional multi-agent DRL algorithms that train synchronously like MAPPO  cannot converge well because the non-stationary time-varying channel causes additional asynchrony between agents. \revb{Our algorithm outperforms the best baseline, F.\ DRL, by about 15.6\% on average.} The performance gap between most algorithms narrows at lower channel bandwidth because the transmission time increases, the strategy of offloading to edges cannot work well compared to processing locally.}

\edit{Additionally, we compare the computational requirements of our proposed A. F. MADRL algorithm with those of the other baseline algorithms: Non-F. MADRL, H-MAPPO, and VarlenMADRL. These experiments were performed under identical system and network settings, including 20 agents, a batch size of 16, and the use of hybrid action spaces. The results are summarized in Table \ref{tab:complexity}.}
\begin{table}[t]
    \centering
    \edit{\caption{Computational Requirements of Different Schemes}
    \label{tab:complexity}
\setlength{\tabcolsep}{3pt} 
\begin{tabular}{lcccc}  
    \toprule
    & A. F. MADRL & Non-F. MADRL & H-MAPPO  &VarlenMADRL\\ 
    \midrule
    MFLOPs& 17.86& 16.39& 10.90 &15.65\\
    \bottomrule
\end{tabular}}
\end{table}

\edit{As shown in Table \ref{tab:complexity}, the computational cost of A. F. MADRL, measured in MFLOPs, is slightly higher than that of the baseline algorithms, yet only marginally above VarlenMADRL ($1.14\times$). This increased cost is accompanied by significant AoI reduction gains, as demonstrated in Fig.~\ref{fig:compare}. This trade-off is acceptable for real-time applications. }

\section{Conclusion}
In this paper, we proposed a unified framework for joint task updating and offloading to minimize AoI in MEC systems modeled as an SMG. \revb{First, we introduced a fractional RL framework for the single-agent scenario with fractional objectives.} Second, we developed a fractional MARL framework for multi-agent problems with fractional objectives. We demonstrated that this framework theoretically converges to Nash equilibrium by showing that it is equivalent to an inexact Newton's method.  Third, we presented an asynchronous trajectory collection mechanism to address asynchronous multi-agent problems in the SMG.  Evaluation results show significant performance improvements in average AoI for both the fractional framework and asynchronous trajectory collection mechanism.  

\revbb{Future work will explore broader AoI-related problems, such as timely multi-task remote inference \cite{shisher2025computation}, and other fractional MEC objectives, including energy consumption and security satisfaction in MEC systems. We will also extend the fractional MARL framework to strongly coupled or highly heterogeneous MEC deployments and study price-of-anarchy-type efficiency bounds for the resulting Nash equilibria.}

\bibliographystyle{IEEEtran}
\bibliography{IEEE}
\appendices
\setcounter{section}{0} 
\renewcommand{\thesection}{\Alph{section}}

\section{Time Complexity of Inner loop in FQL Algorithm }
\begin{proposition}\label{P1}
The proposed FQL Algorithm satisfies the stopping condition $\epsilon_i<-\alpha Q_i(\bs_0,\ba_i)$ for some $\alpha\in(0,1)$, then
after
\begin{equation} \begin{aligned}
    T_i=\left\lceil\frac{11.66 \log(2|\mathcal{Z}|/(E\zeta))}{\alpha^2}\right\rceil
\end{aligned} \end{equation}
steps of SQL in \cite{SQL}, the uniform approximation error $\norm{Q^*_{\gamma_i}-Q_i}\leq \epsilon_i$ holds for all $i\in[E]$, with a probability of $1-\zeta$ for any $\zeta\in(0,1)$.
\end{proposition}
\noindent{\textit{Proof Sketch}:} The main proof of Proposition \ref{P1} is based on \cite{SQL}. Specifically, we can prove that the required $T_i$ is proportional to ${Q_i^2}/{\epsilon_i^2}$, and hence is bounded above by a constant. 

Proposition \ref{P1} shows that\deleting{, surprisingly,} the required number of steps $T_i$ does not increase in $i$, even though the stopping condition $\epsilon_i<- \alpha Q_i(\bs_0,\ba_i)$ is getting more restrictive as $i$ increases.

\renewcommand{\thelemma}{C.\arabic{lemma}}
\renewcommand{\thetheorem}{C.\arabic{theorem}}

\section{Proof of Theorem \ref{T1}}
\label{Proof-T1}

Define 
\begin{subequations}
\begin{equation} \begin{aligned}
   \ba^*(\gamma) &\triangleq \arg\min_{a\in\mathcal{A}}\left(N_{\gamma}(\bs_0,\ba)-\gamma D_{\gamma}(\bs_0,\ba)\right), \\
    Q({\gamma})& \triangleq \min_{\ba\in\mathcal{A}}\left(N_{\gamma}(\bs_0,\ba)-\gamma D_{\gamma}(\bs_0,\ba)\right),\\
    N(\gamma)&=N_\gamma(s_0,\ba^*(\gamma))~~{\rm and}~~D(\gamma)=D_\gamma(s_0,\ba^*(\gamma)),\\
    P(\gamma)&=\frac{N(\gamma)}{D(\gamma)},\\
       \ba_i &\triangleq \arg\min_{\ba\in\mathcal{A}}\left(N_{i}(\bs_0,\ba)-\gamma D_{i}(\bs_0,\ba)\right),
\end{aligned} \end{equation}
\end{subequations}
for all $\gamma\geq 0$.
In the remaining part of this proof, we use $Q_i=Q_i(\bs_0,\ba_i)$, $N_i=N_i(\bs_0,\ba_i)$, and $D_i=D_i(\bs_0,\ba_i)$ for simplicity of presentation.
Note that
\begin{equation} \begin{aligned}
    &\frac{N(\gamma')}{D(\gamma')}-\frac{N_i}{D_i}\label{Proof-Eq1}\\
    \overset{(a)}{\geq}~&\frac{N(\gamma')}{D(\gamma')}-\frac{N(\gamma')}{D_i}-\gamma\left[\frac{D(\gamma')}{D_i}-\frac{D(\gamma')}{D(\gamma)}\right]-\frac{\epsilon_i}{D_i}\nonumber\\
    =~&[-Q(\gamma')+(\gamma_i-\gamma')D({\gamma'})]
    \left(\frac{1}{D_i}-\frac{1}{D(\gamma')}\right)-\frac{\epsilon_i}{D_i},\nonumber
\end{aligned} \end{equation}   
where (a) is from the suboptimality of $N_i$.
Sequences $\{\gamma_i\}$, $\{Q_i\}$, and $\{D_i\}$ generated by FQL Algorithm satisfy
\begin{equation} \begin{aligned}\label{Qea}
    \gamma_i-\gamma^*\geq \gamma_i-{N_i}/{D_{i}}=-{Q_i}/{D_i}>  {\epsilon_i}/{\alpha D_i},
\end{aligned} \end{equation}
for all $i$ such that $Q_i<0$. In addition, from the fact that
    $N_i-\gamma_i D_i\leq  N(\gamma^*)-\gamma_i D(\gamma^*)$ and 
    $N(\gamma^*)-\gamma^* D(\gamma^*)\leq  N_i-\gamma^* D_i$
for all $i\in[E]$, it follows that $
   (\gamma_i-\gamma^*)(D(\gamma^*)-D_i)\leq 0,~\forall i\in[E],$
and hence $ \gamma_i\geq \gamma^*~~{\rm if~and~only~if}~D_i\geq D(\gamma^*).$
It follows from \eqref{Proof-Eq1} that, for all $i$,
\begin{equation} \begin{aligned}
    \gamma_{i+1}-\gamma^*=P_i-\gamma^*\leq~~& (\gamma_i-\gamma^*)\left(1-\frac{D(\gamma^*)}{D_i}\right)+\frac{\epsilon_i}{D_i}\nonumber\\
    \overset{(b)}{<} ~~& (\gamma_i-\gamma^*)\left(1+\alpha-\frac{D(\gamma^*)}{D_i}\right).\nonumber
\end{aligned} \end{equation}
Note that $(b)$ follows by induction from \eqref{Qea}. Specifically, if $\gamma_{i}< \gamma^*$, then $\gamma_{i+1}< \gamma^*$. Therefore, we have that, if $Q_{i}<0$ then $Q_{i+1}<0$, and hence that $\epsilon_i\leq -\alpha Q_i$ for all $i\in[E]$.
Let $D_{max} =  \max_{\mathbf{s}, \mathbf{a}} \{\frac{c_{N}(\mathbf{s}, \mathbf{a})} {c_{D}(\mathbf{s}, \mathbf{a})}\}, D_{min} =  \min_{\mathbf{s}, \mathbf{a}} \{\frac{c_{N}(\mathbf{s}, \mathbf{a})} {c_{D}(\mathbf{s}, \mathbf{a})}\}$. We have $\frac{D_{min}}{D_{max}}\in(0,1)$. Therefore, we can find an  $\alpha\in(0,\frac{D_{min}}{D_{max}})$, such that $(\gamma_{i+1}-\gamma^*)/(\gamma_{i}-\gamma^*)\in(0,1)$ for all sufficiently large $i$, which implies that $\{\gamma_i\}$ converges linearly to $\gamma^*$.

\section{Approximating the Nash Operator} \label{approximate_nash}

One can use the iterative best response algorithm to refine a joint action from an initial guess $\ba^{(0)}$ towards an approximate Nash operator.
In each Q-learning iteration $k$, upon reaching a new state $\bs$, each agent $m$ sequentially updates its action to its best response, maximizing its individual Q-function given the other agents' actions from the previous iteration $\ba_{-m}^{(k-1)}$. This best response $\ba_m^{(k)}$ for agent $m$ is computed as:

\begin{align}
    \ba_m^{(k)} = \arg\min_{\ba_m} Q_{\gamma_m}^k(\bs', \ba_m, \ba_{-m}^{(k-1)}),
    \label{eq:best_response}
\end{align}
where $Q_{\gamma_m}^k$ denotes the Q-function of agent $m$ at iteration $k$. The joint action is then updated to $\ba^{(k)} = (\ba_1^{(k)},..., \ba_m^{(k)},..., \ba_M^{(k)})$, incorporating the newly computed best response. After a few steps of this iterative refinement, the resulting joint action $\ba^{(k)}$ serves as our approximation of the Nash equilibrium in state $\bs'$. Consequently, we can approximate $\mathop{\mathbb{N}}\limits_{\ba\in\mathcal{A}}\bQ^k_{\bgamma}(\bs', \ba)$ with $\bQ^k_{\bgamma}(\bs', \ba^{(k)})$, which is then used to update the Q-value in Equation \eqref{NQ-function}. This iterative procedure allows agents to converge towards a Nash equilibrium by sequentially optimizing their actions in response to the actions of others, effectively approximating the action of the Nash operator.

\section{Proof of Lemma \ref{l_eqv}} \label{lemma2_p}

The Bellman equation form of $V_{\gamma_m}$ of mobile device $m$ is given by:
\begin{equation} \begin{aligned}
    V_{\gamma_m}(\bs,\pii) =  &\mathbb{E}_{\pii}\left\{c_{N_{m}}(\bs, \ba)- \gamma_{m}c_{D_{m}}(\bs, \ba) \nonumber  \right. \\
    & \left.  +\mathbb{E}_{Pr}[\delta V_{\gamma_m}(\bs',\pii)]]\right\}.\label{f_cost}
\end{aligned} \end{equation}
Let $\mathcal{B}(\mathcal{S})$ denote the Banach space of bounded real-valued functions on $\mathcal{S}$ with the supremum norm. Let $\mathbb{W}_{\pii}: \mathcal{B}(\mathcal{S}) \rightarrow \mathcal{B}(\mathcal{S})$ be the mapping:
\begin{equation} \begin{aligned}
    \mathbb{W}_{\pii} V_{\gamma_m}(\bs, \pii)=&\mathop{\min}\limits_{\pii_m} \;     \mathbb{E}_{\pii}\left\{[c_{N_{m}}(\bs, \ba)- \gamma_{m}c_{D_{m}}(\bs, \ba) \nonumber\right. \\
    & \left.+\delta\mathbb{E}_{Pr}[ V_{\gamma_m}(\bs', \pii)]]\right\}.
\end{aligned} \end{equation}
From \cite{fink1964equilibrium}, we have the following conclusion:
\begin{theorem} (Contraction Mapping of $\mathbb{W}_{\pii}$)
    For any  fixed $\gamma_m$ and  $\bs\in\mathcal{S}$,  $\mathbb{W}_{\pii}$ is a contraction mapping of $\mathbb{R}$. \label{cm}
\end{theorem}
By Theorem \ref{cm},   the sequence $V^0_{\gamma_m}=0$, $V^{k+1}_{\gamma_m}=\mathbb{W}_{\pii}V^{k}_{\gamma_m}$ converges to the optimal cost $V^{\star}_{\gamma_m}(\bs)$:
\begin{equation} \begin{aligned}
    V_{\gamma_m}^{\star}(\bs)=&\mathop{\min}\limits_{\pii_m}\; \mathbb{E}_{\pii}\left\{c_{N_{m}}(\bs, \ba) - \gamma_mc_{D_{m}}(\bs, \ba) \nonumber \right.\\
    &\left.+\delta\mathbb{E}_{Pr}[ V_{\gamma_m}^{\star}(\bs')]\right\}. \label{optimal}
\end{aligned} \end{equation}
\begin{lemma}
$V^{\star}_{\gamma_m}$ is continuous with respect to $\gamma_m$. \label{ctn}
\end{lemma}
\begin{proof}
Define function $f$ as
\begin{equation} \begin{aligned}
    f(\gamma_m, \pii, \bs) =&\mathbb{E}_{\pii}\left\{c_{N_{m}}(\bs, \ba) - \gamma_mc_{D_{m}}(\bs, \ba) \nonumber \right.\\
    &\left.+\delta\mathbb{E}_{Pr}[ V_{\gamma_m}^*(\bs')]\right\}.
\end{aligned} \end{equation}
From \eqref{optimal}, we have
\begin{equation} \begin{aligned}
      V_{\gamma_m}^{\star}(\bs)=\mathop{\min}\limits_{\pii_m}    f(\gamma_m, \pii, \bs).
\end{aligned} \end{equation}
Since $c_{N_{m}}(\bs, \ba)$ and $c_{D_{m}}(\bs, \ba)$ are continuous with respect to $\pii$, $f(\gamma_m, \pii, \bs)$ is therefore continuous with respect to $\pii$ and $\gamma_m$.
Thus, as the set of $\pii$ is a compact set, by Berge's maximum theorem, $V_{\gamma_m}^{\star}(\bs)$ is therefore continuous with respect to $\gamma_m$.
\end{proof}

 Since $c_{N_{m}}(\bs, \ba)$, $c_{D_{m}}(\bs, \ba)$ are continuous with respect to $\gamma_m$, we have the Nash operator $\mathop{\mathbb{N}}\limits_{\ba\in\mathcal{A}}$ is therefore continuous with respect to $\gamma_m$. From \eqref{Nash_ND},  we have $N_{\gamma_m}(\bs)$ and $D_{\gamma_m}(\bs)$ are therefore continuous with respect to  $\gamma_m$.

Define a compact and convex set $\Gamma=[0, \overline{\gamma}]^{M}$, where  $\overline{\gamma} = \max\limits_{\mathbf{s}, \mathbf{a}}\left\{\frac{ c_{N,m}(\mathbf{s}, \mathbf{a})}{ c_{D,m}(\mathbf{s}, \mathbf{a})}\right\}$. We can thus define the mapping   $\mathbb{T}:\Gamma \rightarrow \Gamma$ of $\bgamma$ in Algorithm \ref{FNQL} as:
\begin{equation} \begin{aligned}
\mathbb{T}\bgamma = 
\left[\mathop{\mathbb{E}}\limits_{\bs_0\sim\mu_0}\left[\frac{\overline{N}^\star_{\gamma_1}(\bs_0)}{\overline{D}^\star_{\gamma_1}(\bs_0)}\right], \ldots, \mathop{\mathbb{E}}\limits_{\bs_0\sim\mu_0}\left[\frac{\overline{N}^\star_{\gamma_M}(\bs_0)}{\overline{D}^\star_{\gamma_M}(\bs_0)}\right]\right]^T.
\end{aligned} \end{equation}

\begin{theorem} (Existence of Fixed Point) 
    There exists $\bgamma^* \in \Gamma$ such that $\mathbb{T} \bgamma^* = \bgamma^*$.
\end{theorem}
\begin{proof}
    Since $\Gamma$ is compact and convex, by the Brouwer fixed point theorem, there exists a fixed point $\bgamma^{\star}\in\Gamma$ for the continuous mapping $\mathbb{T}$.
\end{proof}
Therefore, there exists   $\bgamma^{*}$, such that $\mathbb{T}\bgamma^{*}=\bgamma^{*}$. Then we have $\mathop{\mathbb{E}}\limits_{\bs_0\sim\mu_0}\left[\frac{\overline{\bN}_{\bgamma^{*}}(\bs_0) }{\overline{\bD}_{\bgamma^{*}}(\bs_0)}\right] = \bgamma^{*}$. Thus, we have $\bF(\boldsymbol{\theta}^*,\bgamma^*)=\mathop{\mathbb{E}}\limits_{\bs_0\sim\mu_0}\left[\overline{\bN}^{\star}_{\bgamma^{*}}(\bs_0)\right] - \bgamma^{*} \mathop{\mathbb{E}}\limits_{\bs_0\sim\mu_0}\left[\overline{\bD}^{\star}_{\bgamma^{*}}(\bs_0)\right]=0$.

\section{Proof of Theorem \ref{T2}} \label{convergence_p}

By Lemma \ref{G-NE}, for each $\bgamma$,  there exists mapping  $\boldsymbol{\theta}(\bgamma)$ that satisfies the NE conditions. For simplicity, we denote $\boldsymbol{\theta}$ as $\boldsymbol{\theta}(\bgamma)$ and  $\boldsymbol{\theta}_i$ as $\boldsymbol{\theta}(\bgamma_i)$.  Assuming differentiability, these are the first-order necessary conditions where each agent optimizes its objective given the others' parameters:
\begin{equation}
    \begin{aligned}
W_m(\boldsymbol{\theta}, \gamma_m)& := \nabla_{\boldsymbol{\theta}_m} F_m(\boldsymbol{\theta}, \gamma_m)\\
&= \nabla_{\boldsymbol{\theta}_m} N_m(\boldsymbol{\theta}) - \gamma_m \nabla_{\boldsymbol{\theta}_m} D_m(\boldsymbol{\theta}) = \mathbf{0}         
    \end{aligned}
\end{equation} 
for all $m\in\M$. Let $\boldsymbol{W}(\boldsymbol{\theta}, \boldsymbol{\gamma}) = (W_1(\boldsymbol{\theta}, \gamma_1), \dots, W_M(\boldsymbol{\theta}, \gamma_M))$ be the stacked vector of these gradient conditions. The equilibrium condition is thus $W(\boldsymbol{\theta}, \boldsymbol{\gamma}) = \mathbf{0}$. We focus on a specific equilibrium point $\boldsymbol{\theta}$ corresponding to  $\boldsymbol{\gamma}$. From Assumption \ref{as:local_strong_convexity}
, we have $\boldsymbol{W}$ is invertible and $\boldsymbol{\theta}$ is locally a differentiable function of $\boldsymbol{\gamma}$ around $\boldsymbol{\gamma}$. By the Implicit Function Theorem (IFT), its Jacobian is given by:
\begin{align}
     \nabla_{\boldsymbol{\gamma}} \boldsymbol{\theta}(\boldsymbol{\gamma}) = - [\nabla_{\boldsymbol{\theta}} \boldsymbol{W}(\boldsymbol{\theta}, \boldsymbol{\gamma})]^{-1} [\nabla_{\boldsymbol{\gamma}} \boldsymbol{W}(\boldsymbol{\theta}, \boldsymbol{\gamma})]. 
\end{align}

We denote  the Jacobian of $W$ with respect to  $\boldsymbol{\theta}$ as $J_{\bW}(\boldsymbol{\theta}, \boldsymbol{\gamma})$. Its $(m, n)$-th block is given by:
\begin{equation}
    \begin{aligned}
         &(J_{\bW}(\boldsymbol{\theta}, \boldsymbol{\gamma}))_{mn}\\
         =& \nabla_{\boldsymbol{\theta}_n} W_m(\boldsymbol{\theta}, \gamma_m) = \nabla^2_{\boldsymbol{\theta}_n \boldsymbol{\theta}_m} F_m(\boldsymbol{\theta}, \gamma_m) \\
         =&\nabla^2_{\boldsymbol{\theta}_n \boldsymbol{\theta}_m} N_m(\boldsymbol{\theta}) - \gamma_m \nabla^2_{\boldsymbol{\theta}_n \boldsymbol{\theta}_m} D_m(\boldsymbol{\theta}).
    \end{aligned}
\end{equation}
We denote the Jacobian of $W$ with respect to $\boldsymbol{\gamma}$ as  $\nabla_{\boldsymbol{\gamma}} W$. Its $(m, n)$-th block is given by
\begin{align}
     &\frac{\partial W_m}{\partial \gamma_n}(\boldsymbol{\theta}, \boldsymbol{\gamma})\nonumber\\\ =& \frac{\partial}{\partial \gamma_n} \left[ \nabla_{\boldsymbol{\theta}_m} N_m(\boldsymbol{\theta}) - \gamma_m\nabla_{\boldsymbol{\theta}_m} D_m(\boldsymbol{\theta}) \right]\nonumber\\
     =& - \delta_{mn} \nabla_{\boldsymbol{\theta}_m} D_m(\boldsymbol{\theta})
\end{align}
The $n$-th column of $\nabla_{\boldsymbol{\gamma}} W$, denoted $\frac{\partial W}{\partial \gamma_n}(\boldsymbol{\theta}, \boldsymbol{\gamma})$, is given by:
\begin{align}
     \frac{\partial W}{\partial \gamma_n}(\boldsymbol{\theta}, \boldsymbol{\gamma}) = \begin{pmatrix} \mathbf{0} \\ \vdots \\ \mathbf{0} \\ -\nabla_{\boldsymbol{\theta}_n} D_n(\boldsymbol{\theta}) \\ \mathbf{0} \\ \vdots \\ \mathbf{0} \end{pmatrix}. 
\end{align}
Let $((J_{\bW} )^{-1})_{mk}$ denote the $(m, k)$-th block of the inverse matrix $(J_{\bW})^{-1}$. By the IFT formula, the $j$-th block of the sensitivity vector is given by:
\begin{equation}
\begin{aligned}
     \frac{\partial \boldsymbol{\theta}_m}{\partial \gamma_n}& = - \sum_{k=1}^M ((J_{\bW}(\boldsymbol{\theta}, \boldsymbol{\gamma}) )^{-1})_{mk} \left( \frac{\partial W_k}{\partial \gamma_n}(\boldsymbol{\theta}, \boldsymbol{\gamma}) \right)\\
     &=  ((J_{\bW}(\boldsymbol{\theta}, \boldsymbol{\gamma}) )^{-1})_{mn} \nabla_{\boldsymbol{\theta}_n} D_n(\boldsymbol{\theta}).
\end{aligned}    
\end{equation}
We then have the bound as follows.
\begin{lemma}\label{lm:bound_inverse}
    By Assumptions \ref{as:local_strong_convexity} and \ref{as:bounded_interaction}, there exists $\kappa\in(0,1)$, $J_{\bW}$ is invertible, and the block $L_\infty$ norm of its inverse is bounded by:
$$ \|(J_{\bW})^{-1}\|_\infty  \le \frac{\max_k \|((J_{\bW})_{kk})^{-1}\|}{1-\kappa}. $$

\end{lemma}
\begin{proof}
    From Assumption \ref{as:bounded_interaction} and \ref{as:local_strong_convexity},   we have that 
    \begin{equation}
        \begin{aligned}
&\sum_{l \neq k} \|((J_{\bW})^{-1})_{kk})^{-1} (J_{\bW})^{-1})_{kl}\| \\
\leq&\|((J_{\bW})^{-1})_{kk})^{-1}\| \|(J_{\bW})^{-1})_{kl}\| \\
\leq&\frac{KH_{\text{int}}}{\mu} <1.
        \end{aligned}
    \end{equation}
    Thus, there exists $\kappa\in(0,1)$, such that 
    $$\kappa=\frac{KH_{\text{int}}}{\mu}.$$
The block inverse formula gives
\begin{equation}
    \begin{aligned}
      &  ((J_{\bW})^{-1})_{kn} \\
      = &\delta_{kn} ((J_{\bW})_{kk})^{-1} - \sum_{l \neq k} ((J_{\bW})_{kk})^{-1} (J_{\bW})_{kl} ((J_{\bW})^{-1})_{ln}.
    \end{aligned}
\end{equation}
Taking the norm and summing over $n$,
\begin{equation}
    \begin{aligned}
        &\sum_n \|((J_{\bW})^{-1})_{kn}\| \\
        \le& \|((J_{\bW})_{kk})^{-1}\| + \sum_{l \neq k} \|((J_{\bW})_{kk})^{-1} (J_{\bW})_{kl}\|\\
        &\cdot\sum_n \|((J_{\bW})^{-1})_{ln}\|.
    \end{aligned}
\end{equation}
Taking $\max_k$ on both sides,
\begin{equation}
    \begin{aligned}
        &\max_k \sum_n \|((J_{\bW})^{-1})_{kn}\| \\
        \le &\max_k \|((J_{\bW})_{kk})^{-1}\| + \kappa \max_k \sum_n \|((J_{\bW})^{-1})_{kn}\|.
    \end{aligned}
\end{equation}
Thus, we have
$$(1-\kappa) \max_k \sum_n \|((J_{\bW})^{-1})_{kn}\| \le \max_k \|((J_{\bW})_{kk})^{-1}\|,$$
and 
$$\max_k \sum_n \|((J_{\bW})^{-1})_{kn}\| \le \frac{\max_k \|((J_{\bW})_{kk})^{-1}\|}{1-\kappa}.$$
\end{proof}

Using Lemma \ref{lm:bound_inverse} and Assumption \ref{as:local_strong_convexity}, there exists $\kappa\in(0,1)$, such that
\begin{align} \label{eq:inverse}
    \|(J_{\bW})^{-1}\|_\infty \leq \frac{1}{\mu(1-\kappa)}.
\end{align}
We can solve for the residual $r_i$ as:
\begin{equation} \begin{aligned}
    \boldsymbol{r}_i =&\bF(\boldsymbol{\theta_i},\bgamma_i)  +J_{\bF}(\boldsymbol{\theta_i},\bgamma_i) s_i \\
        =& \left[ I- J_{\bF}(\boldsymbol{\theta}_i,\bgamma_i )J_{\bF}'(\boldsymbol{\theta}_i,\bgamma_i)^{-1} \right] F(\boldsymbol{\theta}_i,\bgamma_i).  \label{eq:residual}
        \end{aligned}
\end{equation} 
The Inexact Newton method requires the residual to satisfy $\norm{\boldsymbol{r}_i} \le \eta_i \norm{\bF(\boldsymbol{\theta}_i,\bgamma_i)}$ for some forcing sequence $\{\eta_i\}$. From \eqref{eq:residual}, \eqref{eq:inverse} and Assumption \ref{as:bounded_denominator},  \ref{as:bounded_interaction}, we have that
\begin{equation}
\begin{aligned}\label{eq:forcing_term}
     &\eta_i\\ =& \| I- J_{\bF}(\boldsymbol{\theta}_i,\bgamma_i )J_{\bF}^{\prime}(\boldsymbol{\theta}_i,\bgamma_i)^{-1}\|_\infty \\
     \leq&\max_m\sum_{n \neq m}  \frac{1}{|D_n(\boldsymbol{\theta}_i)|} \sum_{k \in \mathcal{N}_m^{grad}} \|\nabla_{\boldsymbol{\theta}_{k,i}} F_m(\boldsymbol{\theta}_i, \gamma_{m,i})\|_\infty \\
     & \cdot \|(J_{\bW}(\boldsymbol{\theta}_i,\bgamma_i ) ^{-1})_{kn}\|_\infty  \cdot \|\nabla_{\boldsymbol{\theta}_n} D_n(\boldsymbol{\theta}_i)\|_\infty  \\
    \leq& \max_m \frac{C_D C_{\text{int}}}{D_{min}} \sum_{k \in \mathcal{N}_m^{grad}} \left( \sum_{n \neq m} \|(J_{\bW} (\boldsymbol{\theta}_i,\bgamma_i )^{-1})_{kn}\|_\infty  \right) \\
   \leq & \frac{K C_D C_{\text{int}}}{\mu D_{min} (1-\kappa)}=\frac{C_D C_{\text{int}}}{ D_{\text{min}} (1/K- H_{\text{int}})}.
\end{aligned}
\end{equation}

Thus, there exists $\eta_{max}= \frac{C_D C_{\text{int}}}{ D_{\text{min}} (1/K- H_{\text{int}})}<1$ such that $\eta_i\leq\eta_{max}$.

Therefore, under the stated assumptions, there exists a neighborhood $\mathcal{U}' \subseteq \mathcal{U}(\bgamma^*)$ such that if $\bgamma^0 \in \mathcal{U}'$, the sequence $\{\bgamma_i\}$  converges to $\bgamma^*$. The outer-loop iteration in Algorithm \ref{FNQL} converges to NE of game $\bG$ linearly.

\section{ASYNCHRONOUS FRACTIONAL MULTI-AGENT DRL NETWORKS} \label{a_networks}
As shown in Fig. \ref{fig:MADRL}, for each mobile device, our algorithm contains RNN-D3QN and RNN-PPO networks for discrete and continuous action spaces as well as interactions between agents and history information.
\subsection{RNN-D3QN Module}
D3QN is a modified version of deep Q-learning \cite{mnih2015human} equipped with double DQN technique with target Q to overcome overestimation \cite{van2016deep} and dueling DQN technique which separately estimates value function and advantage function \cite{wang2016dueling}. Our RNN-D3QN architecture is inspired by DRQN \cite{hausknecht2015deep} which allows networks to retain context and memory from history events and environment states. The key idea is to learn a mapping from states in state space $S^O$ to the Q-value of actions in action space $A^O$. When the policy is learned and updated to the mobile device, it can select offloading actions with the maximum Q-value to maximize the expected long-term reward. Apart from traditional deep Q-learning, it has a target Q-function network to compute the expected long-term reward based on the optimal action chosen with the traditional Q-function network to solve the overestimation problem. Moreover, both the Q-function networks contain the advantage layer and value layer, which are used to estimate the value of the state and the relative advantage of the action. In addition, we take advantage of a GRU module to provide history information of other agents' decisions and environment states.

\subsection{RNN-PPO Module}
RNN-PPO mainly includes two networks: a policy network and a value network. The policy network is responsible for producing actions. It outputs mean and standard deviation vectors for continuous actions and a probability distribution over discrete actions. The value network computes the advantage values and estimates the expected return of the current state. Similar to the RNN-D3QN network, it has an additional GRU module to memorize the state-action pairs from history and other agents to enhance the information contained by the state based on which the value network estimates the state value.

\begin{figure}[t]
	\centering 
 	\includegraphics[width=8.5cm]{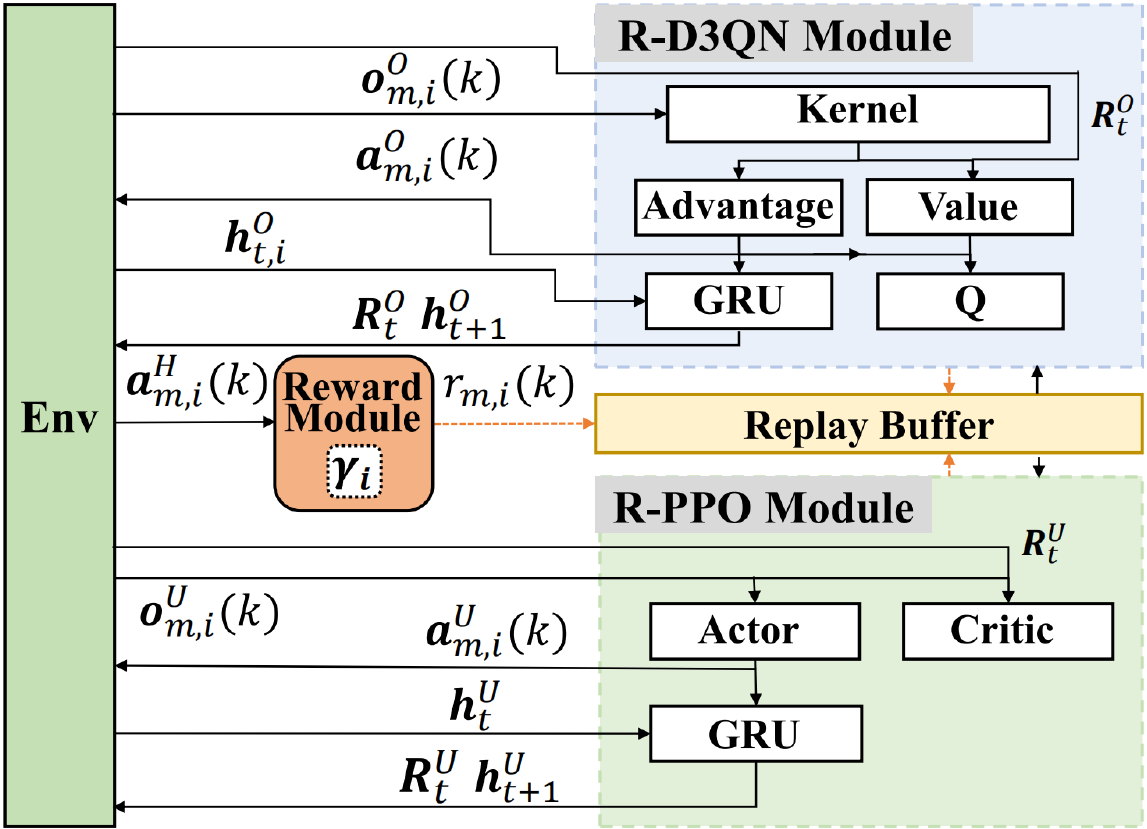}
	\caption{\revbb{Illustration of the proposed RNN-D3QN and RNN-PPO architecture. The GRU-based history encoder is used during centralized training to model asynchronous interaction history, and does not imply an online centralized input during decentralized execution. }}
 \label{fig:MADRL}
\end{figure}


\section{\textbf{Experimental Setup and Hyperparameters} }

\revb{To ensure the reproducibility of our reported results, this appendix provides a comprehensive list of the environment settings, communication model parameters, algorithmic hyperparameters, and training details used in our simulations.}

\subsection{Environment and Task Parameters}
\revb{The default parameters for the Mobile Edge Computing (MEC) environment and the computational tasks are detailed in Table~\ref{tab:env_params}. These settings were used for all experiments unless explicitly varied, as noted in the main text.}

\revb{%
\begin{table}[ht]
\centering
\caption{Default Environment and Task Parameter Settings}
\label{tab:env_params}
\begin{tabular}{@{}ll@{}}
\toprule
\textbf{Parameter}                   & \textbf{Value}                     \\ \midrule
Number of mobile devices             & 20                                 \\
Number of edge nodes                 & 2                                  \\
Capacity of mobile devices           & 2.5 GHz                            \\
Capacity of edge nodes               & 41.8 GHz                           \\
Task size                            & 30 Mbit                            \\
Task density                         & 0.297 gigacycles per Mbit         \\
Drop coefficient ($\bar{Y}$)         & 1.5                                \\ \bottomrule
\end{tabular}
\end{table}
}

\subsection{Communication Model Parameters}
\revb{The general parameters for the wireless communication model, as described in Equation (1), are listed below.}
\revb{%
\begin{itemize}
    \item Transmit power ($p_m$): 20 dBm
    \item Background noise power ($\eta_0$): -114 dBm
    \item Path loss exponent ($\alpha$): 3
\end{itemize}
}
\revb{For the specific experiment involving time-varying channels (Fig.~6), the parameters were set according to Table~\ref{tab:channel_params}.}

\revb{%
\begin{table}[ht]
\centering
\caption{Channel Parameter Settings for Time-Varying Experiments (Fig. 6)}
\label{tab:channel_params}
\begin{tabular}{@{}ll@{}}
\toprule
\textbf{Parameter}              & \textbf{Value}    \\ \midrule
Number of mobile devices        & 10                \\
Transmit power                  & 20 dBm            \\
Background noise power          & -114 dBm          \\
Path loss exponent              & 3                 \\ \bottomrule
\end{tabular}
\end{table}
}

\subsection{Algorithmic Hyperparameters for A.F. MADRL}
\revb{The performance of our proposed Asynchronous Fractional Multi-Agent DRL (A.F. MADRL) algorithm depends on a set of key hyperparameters. Table~\ref{tab:hyperparams} provides the specific values used for our implementation, which are essential for replicating our learning-based results.}

\revb{%
\begin{table}[ht]
\centering
\caption{Hyperparameters for the A.F. MADRL Algorithm}
\label{tab:hyperparams}
\resizebox{\columnwidth}{!}{%
\begin{tabular}{@{}lll@{}}
\toprule
\textbf{Category} & \textbf{Hyperparameter} & \textbf{Value} \\ \midrule
\textbf{General RL} & & \\
 & Discount Factor ($\delta$) & 0.99 \\
 & Replay Buffer Size & 100,000 \\
 & Batch Size & 64 \\
 & Target Network Update Rate ($\tau$) & 0.005 (soft update) \\ \midrule
\textbf{RNN-D3QN (Offloading)} & & \\
 & Learning Rate & 1e-4 (Adam Optimizer) \\
 & $\epsilon$-Greedy Schedule & Linear decay from 1.0 to 0.05 over 500 episodes \\ \midrule
\textbf{RNN-PPO (Updating)} & & \\
 & Actor Learning Rate & 3e-4 (Adam Optimizer) \\
 & Critic Learning Rate & 1e-3 (Adam Optimizer) \\
 & PPO Clipping Parameter ($\epsilon$) & 0.2 \\
 & Entropy Coefficient & 0.01 \\ \midrule
\textbf{Network Architecture} & & \\
 & GRU Hidden Units & 128 \\
 & Fully Connected Layers (Post-GRU) & Two layers with [256, 128] units and ReLU activation \\ \midrule
\textbf{Fractional Module} & & \\
 & $\gamma$ Update Frequency & Every 50 episodes \\ \bottomrule
\end{tabular}%
}
\end{table}
}

\subsection{Training Details}
\revb{All algorithms were implemented using PyTorch. The training was conducted on a server equipped with an AMD EPYC 7763 CPU and an NVIDIA RTX 4090 GPU. To ensure the stability and reliability of our results, we used fixed random seeds for network initialization and environment generation. Each experiment was run for 1000 episodes (or 1500 if convergence was slower), and the final reported results are the average of five independent runs, with the shaded areas in plots representing the standard deviation across these runs.}

\end{document}